\documentclass[11pt]{article}

\usepackage[margin=1in]{geometry}

\usepackage{microtype}
\usepackage{graphicx}
\usepackage{subcaption}
\usepackage{booktabs} 
\usepackage{xcolor}
\usepackage{colortbl}
\usepackage{bm}
\usepackage[numbers]{natbib}  
\usepackage{float}

\usepackage{hyperref}
\usepackage[noblocks]{authblk}  
\usepackage{titletoc}  

\usepackage{amsmath}
\usepackage{amssymb}
\usepackage{mathtools}
\usepackage{amsthm}
\usepackage{amssymb}

\usepackage{thmtools}

\usepackage[normalem]{ulem}

\usepackage{cleveref}
\crefname{equation}{Eqn.}{Eqns.}
\Crefname{equation}{Eqn.}{Eqns.}

\newcommand{\E}{\mathbb{E}}

\newcommand{\R}{\mathbb{R}}
\newcommand{\supp}{\mathrm{supp}}

\newcommand{\sh}[2]{S_{#1}({#2})}

\newcommand{\pd}{p_\mathcal{D}}

\theoremstyle{plain}
\newtheorem{theorem}{Theorem}[section]

\newtheorem{lemma}[theorem]{Lemma}
\newtheorem{corollary}[theorem]{Corollary}
\theoremstyle{definition}
\newtheorem{definition}[theorem]{Definition}

\theoremstyle{remark}
\newtheorem{remark}[theorem]{Remark}

\usepackage[textsize=tiny]{todonotes}

\title{Two Calm Ends and the Wild Middle:\\A Geometric Picture of Memorization in  Diffusion Models}

\author[1]{Nick Dodson$^{*}$}
\author[1]{Xinyu Gao$^{*}$}
\author[2]{Qingsong Wang$^{*}$}
\author[2]{Yusu Wang$^{\dagger}$}
\author[1]{Zhengchao Wan$^{\dagger\ddagger}$}

\affil[1]{Department of Mathematics, University of Missouri, Columbia, Missouri, USA}
\affil[2]{Halıcıoğlu Data Science Institute, University of California San Diego, La Jolla, California, USA}

\date{\vspace{-1em}\small $^{*}$Co-first authors. $^{\dagger}$Co-last authors. $^{\ddagger}$\texttt{zwan@missouri.edu}}

\begin{document}

\maketitle

\begin{abstract}
Diffusion models generate high-quality samples but can also memorize training data, raising serious privacy concerns. Understanding the mechanisms governing when memorization versus generalization occurs remains an active area of research. In particular, it is unclear where along the noise schedule memorization is induced, how data geometry influences it, and how phenomena at different noise scales interact. We introduce a geometric framework that partitions the noise schedule into three regimes based on the coverage properties of training data by Gaussian shells and the concentration behavior of the posterior, which we argue are two fundamental objects governing memorization and generalization in diffusion models. This perspective reveals that memorization risk is highly non-uniform across noise levels. We further identify a danger zone at medium noise levels where memorization is most pronounced. In contrast, both the small and large noise regimes resist memorization, but through fundamentally different mechanisms: small noise avoids memorization due to limited training coverage, while large noise exhibits low posterior concentration and admits a provably near linear Gaussian denoising behavior.
For the medium noise regime, we identify geometric conditions through which we propose a geometry-informed targeted intervention that mitigates memorization.
\end{abstract}

\section{Introduction}
Diffusion models have achieved remarkable success in generative modeling, producing high-quality images but sometimes memorizing the training data, raising serious privacy and copyright concerns~\citep{somepalli2023diffusion, carlini2023extracting, wen2024detecting}.
This memorization risk is inherent to the training objective, as it admits a unique global minimum, namely the \emph{empirical optimal} solution, which can only generate the training data \cite{gu2024on}. 
Consequently, it is both theoretically interesting and practically important to understand how and when a learned model can deviate from the empirical optimal to generalize beyond its training data.

Some recent works have analyzed this question. In particular, \cite{vastola2025generalization} argues that the stochastic nature of diffusion model training helps prevent fitting the empirical optimal denoiser, and thus promotes generalization.
Later, \cite{bertrand2025closed} contests this view, showing that for a large range of training times the regression target is essentially unique while model still generalizes.
Recently, \cite{song2025selective} point out that supervised training in diffusion models is limited to small regions which fails to cover most parts of the inference trajectory, and argues that the generalization behavior of the model stems from extrapolation.

\paragraph{Our work.}
A key aspect to understand when models \emph{generalize} is to understand when they \emph{memorize}, which we focus on in this paper.
To this end, we distinguish between \emph{trajectory-level memorization}, which concerns whether full sampling trajectories collapse to the training set, and \emph{per-noise-level memorization}, which characterizes memorizing denoiser behavior at a fixed noise scale.
Empirically, we find that models exhibiting strong trajectory-level memorization may nonetheless behave benignly at many individual noise levels, indicating that memorization is not induced uniformly across the noise schedule (see~\Cref{sec:empirical_investigation} for more details).

This observation motivates our central question: \emph{At which noise levels is memorization formed, and what geometric mechanisms govern its emergence?}

To address these questions, we analyze diffusion training through a geometric lens based on two quantities: \emph{posterior weight concentration} and \emph{Gaussian shell coverage} of the training data. This perspective naturally partitions the noise schedule into three regimes. At both small and large noise levels, the training lies on ``calm ends'' that are resistant to memorization, but for fundamentally different reasons: in the small noise regime, memorization is suppressed due to limited training coverage, while in the large noise regime due to weak posterior concentration and effectively linear denoising behavior. In contrast, a ``danger zone'' emerges at intermediate noise level -- the ``wild middle'' -- where these two effects align and memorization risk peaks sharply. Crucially, memorization formed in this regime can propagate along the inference trajectory and persist into the small noise stage, leading the model to memorize even when both ends of the schedule remain benign. This observation suggests that, to avoid memorization, it is most effective to control the behavior of diffusion models specifically within this intermediate ``danger zone''. By analyzing how coverage and posterior concentration interact, we can predict the location of this danger zone directly from the dataset. Experiments in Section~\ref{sec:experiments} confirm this picture and show that selectively undertraining the intermediate regime can substantially mitigate memorization while preserving generation quality.

\paragraph{More related work.}
Several explanations have been proposed for diffusion model generalization. Notably, \cite{li2024understanding} shows that models in the generalization regime exhibit closeness to a Gaussian linear model, suggesting an inductive bias toward learning the covariance structure.
Neural networks also exhibit spectral bias \citep{rahaman2019spectral, wang2025analytical, kadkhodaie2024generalization}, which may influence what diffusion models learn. Other factors known to promote generalization include early stopping before memorization occurs, limiting model capacity relative to dataset size~\citep{li2024understanding}, increasing the amount of training data as the scaling behavior for generalization and memorization differ in terms of the number of data~\citep{bonnaire2025why}, and conditional training with auxiliary noise classes to provide distributional contrast~\citep{yoon2023diffusion}.

\section{Background}\label{sec:bg}
Let $p$ denote a data distribution on $\R^d$. Let $\bm{X}\sim p$ denote a random variable following $p$ and $\bm{Z}\sim\mathcal N(0,I_d)$ a standard Gaussian random variable independent of $\bm{X}$. For the training and sampling of diffusion models, we mostly follow the description given in \cite{karras2022elucidating}.

\paragraph{Training.} In diffusion models, one of the central tasks is to denoise a noisy observation $\bm{X}_\sigma=\bm{X}+\sigma \bm{Z}$ for various noise levels $\sigma>0$ in order to obtain $\bm{X}$. In particular, one popular diffusion training loss is defined as
\begin{equation}\label{eq:training loss}
    \mathcal L(\theta) := \mathbb E_{\substack{\bm{X} \sim p\\\bm{Z} \sim \mathcal N (0,I_d)\\\sigma \sim \lambda}}\left\|m^{\theta}_\sigma(\bm{X}_\sigma) - \bm{X}\right\|^2
\end{equation}
where $\lambda$ denotes certain distribution on some chosen $[\sigma_{\min},\sigma_{\max}]\subset(0,\infty)$, and $m_\sigma^\theta:\R^d\to\R^d$ is a neural network parameterized by $\theta$ that aims to approximate the \emph{(optimal) denoiser} $m_\sigma$ that we will introduce below. We refer to $m_\sigma^\theta$ as a \emph{trained denoiser} in what follows.

\paragraph{Optimal solution.} The above loss admits a unique global minimum given by the posterior mean which we call the \emph{optimal denoiser}:
$m_\sigma(x):=\mathbb E[\bm X\mid \bm X_\sigma=x]$.
In practice, a model is trained with finite samples $\mathcal{D}=\{x_i\}_{i=1}^N$ assumed to be drawn i.i.d. from some data distribution $p$. In this case, the expectation over $\bm X\sim p$ in the training loss \Cref{eq:training loss} can be replaced by $\bm X\sim \pd$ where $\pd := \frac{1}{N}\sum_{i=1}^N \delta_{x_i}$, and the resulting optimal denoiser, which we now call the \emph{empirical optimal denoiser}, has a closed form given by
\begin{equation}\label{eq:empirical optimal denoiser}
    m_\sigma(x) = \sum_{i=1}^N w_i(x,\sigma)x_i
\end{equation}
where $w_i(x,\sigma) = \mathbb P(\bm{X}= x_i \mid \bm{X}_\sigma = x)$ is called the \emph{posterior weight} on $x_i$ at the noise level $\sigma$ and is given by
\begin{equation}\label{eq:posterior weight}
    w_i(x,\sigma) = \frac{\exp\left(-\frac{\|x - x_i\|^2}{2\sigma^2}\right)}{\sum_{j=1}^N \exp\left(-\frac{\|x - x_j\|^2}{2\sigma^2}\right)}.
\end{equation}

\noindent \textbf{Sampling.} After training, one samples $z\sim \mathcal{N}(0,I_d)$ and lets $x_{\sigma_{\max}}=\sigma_{\max}z$ for some large $\sigma_{\max}$. New samples are generated by solving the following ODE:
\begin{equation}\label{eq:denoiser_ode}
    {dx_\sigma}/{d\sigma} = -({m_\sigma^\theta(x_\sigma) - x_\sigma})/{\sigma}
\end{equation}
from $\sigma_{\max}$ to $\sigma_{\min}$. For any $\sigma\in [\sigma_{\min},\sigma_{\max}]$, we denote by $\Psi^\theta_{\sigma}:\R^d\to\R^d$ the flow map induced by the denoiser $m_\sigma^\theta$, which maps $x_{\sigma_{\max}}$ to $x_\sigma$ by solving the ODE above.

\begin{figure}[t]
    \centering
    \includegraphics[width=0.95\linewidth]{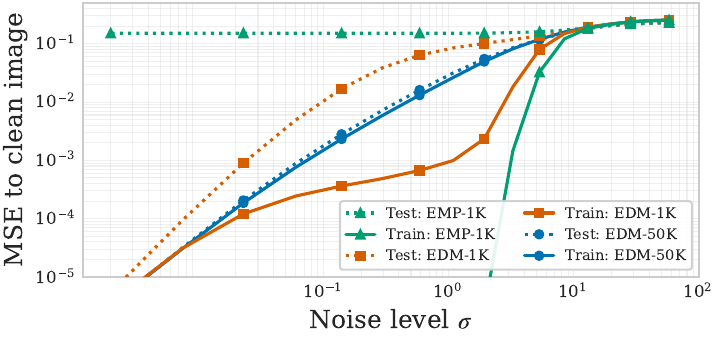}
    \caption{\textbf{MSE to Clean Image.} Comparison of denoising quality across noise levels. Solid lines: training data; dotted lines: test data. \textbf{EDM-1K} shows a generalization gap in the mid-$\sigma$ region.}
    \label{fig:mem_mse_to_clean}
\end{figure}

\section{Empirical Investigation of Memorization}
\label{sec:empirical_investigation}
To understand where and how memorization emerges during diffusion model training, we begin with an empirical investigation that motivates our subsequent geometric framework. By comparing models with different memorization behaviors across the noise schedule, we find the intermediate noise level is the most susceptible to memorization.

We start by comparing trained denoisers and the empirical optimal denoiser on the CIFAR-10 dataset \cite{krizhevsky2009learning}. We adopt the EDM from \cite{karras2022elucidating} to model $m_\sigma^\theta$ and consider three denoisers: a full-data EDM denoiser trained on 50k images (\textbf{EDM-50K}),
an EDM denoiser trained on only 1k images (\textbf{EDM-1K})\footnote{ Training on 1k CIFAR-10 images is known to lead to memorization \cite{gu2024on}. We also validate this in \Cref{tab:swap_memorization}. The 1k data is the first 1k in CIFAR-10 which is randomly indexed.}, 
and the empirical optimal denoiser under the same 1k training images (\textbf{EMP-1K}).

One straightforward way to compare the three models is to evaluate their denoising losses via \Cref{eq:training loss}. For a fixed noise level $\sigma$, this objective reduces to the denoising mean squared error (MSE). We therefore evaluate the fixed-$\sigma$ denoising MSE of the three models over $\sigma\in[0.002,80]$ on both training and test data from CIFAR-10:
\begin{equation}\label{eq:denoising mse}
    \mathrm{MSE}_\sigma(p_\bullet) := \mathbb E_{\substack{\bm{X} \sim p_\bullet\\\bm{Z} \sim \mathcal N (0,I_d)}}\left\|m_{\sigma}^\bullet(\bm{X}_\sigma) - \bm{X}\right\|^2
\end{equation}
where $m_\sigma^\bullet$ denotes either a trained or an optimal denoiser and $p_\bullet$ is either the training distribution $\pd$ or the test distribution $p_\mathcal{T}$. 

Intuitively,  $\mathrm{MSE}_\sigma(\pd)$  measures how well a model can recover a training image from its noisy version at noise level $\sigma$, while  $\mathrm{MSE}_\sigma(p_\mathcal{T})$ measures the corresponding performance on unseen test images. Taken together, these two quantities provide a coarse diagnostic of generalization behavior across noise scales: a small train–test gap suggests similar behavior on training and test data, whereas a large gap indicates degradation on test samples and may signal memorization. We emphasize, however, that train–test MSE alone is not a definitive measure of memorization, but serves here as a useful summary statistic.

\begin{figure}[t]
    \centering
    \includegraphics[width=0.95\linewidth]{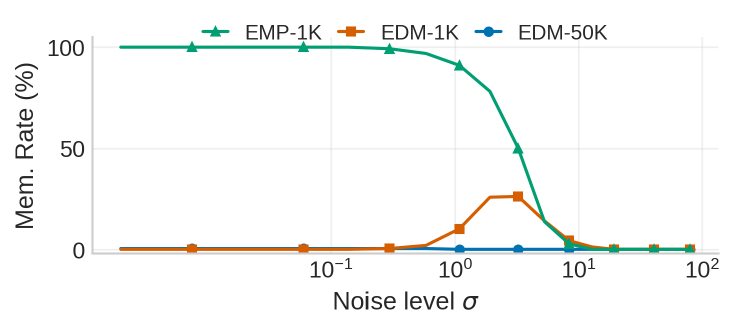}
    \caption{\textbf{Per-Noise-Level Memorization Rate.} Fraction of denoised test images classified as memorized at each noise level.}
    \label{fig:mem_memorization_rate}
\end{figure}
The results are shown in \Cref{fig:mem_mse_to_clean}. Interestingly, we observe three distinct noise regions in which the three models exhibit qualitatively different behavior. At this stage, this partition is purely observational; in subsequent sections, we provide a geometric characterization that explains the emergence of these large-, intermediate-, and small noise regimes.
\begin{enumerate}
    \item At large $\sigma$ (i.e., large noise), all three denoisers behave similarly on both train and test. We will later explain in \Cref{early_stage} that this is due to the fact that high noise denoising is dominated by coarse statistics.

    \item At middle $\sigma$, \textbf{EMP-1K} sharply diverges from \textbf{EDM-50K}, achieving near-zero training error but large test error due to its nearest-neighbor nature, while \textbf{EDM-50K} continues to generalize well.
    \item Interestingly, at small $\sigma$, although \textbf{EDM-1K} generates training data during sampling, it differs from \textbf{EMP-1K} (the empirical optimal denoiser) in its per $\sigma$-level denoising behavior on test data. 
\end{enumerate}

This three-regime pattern motivates a closer examination of memorization for the three denoisers. In particular, instead of only considering the classic notion of memorization, we should also consider memorization risk {\it at each noise level}. While the MSE gap between train and test data may be a good metric for generalization, it is not directly capturing memorization. To resolve this issue, we propose the following metrics for evaluating memorization.

\paragraph{Trajectory-Level memorization.} For a denoiser $m^\bullet_\sigma$ with noise schedule $\sigma\in[\sigma_{\text{min}}, \sigma_{\text{max}}]$, let $\Psi^\bullet_\sigma$ be the resulting flow map given by solving \Cref{eq:denoiser_ode}. Intuitively, the denoiser $m^\bullet_\sigma$ is said to have memorized the data set $\mathcal D$ if its generated samples through the flow map are close to one of the data points in $\mathcal D$. While there could be many ways of quantifying this intuition, we adopt the idea from \cite{yoon2023diffusion}, in which they proposed to classify a generated sample as memorized if it is significantly closer to its nearest training neighbor than to the second nearest. Formally, we say $\{m^\bullet_\sigma\}_{\sigma\in[\sigma_{\min},\sigma_{\max}]}$ is memorizing if for a large portion of noise samples $\bm{Z} \sim N(0,I_d)$, $d_{1\mathrm{NN}}(\Psi^\bullet_{\sigma_{\text{min}}}(\sigma_{\text{max}}\bm{Z})) < d_{2\mathrm{NN}}(\Psi^\bullet_{\sigma_{\text{min}}}(\sigma_{\text{max}}\bm{Z})/3$, where $d_{K\mathrm{NN}}$ denotes the distance to the $K$th nearest point in $\mathcal{D}$. In particular, we compute the ratio of samples satisfying this condition to quantify the trajectory-level memorization rate.

\begin{table}[t]
\centering
\caption{\textbf{Memorization Rates in Denoiser Swapping.} Memorization measured as fraction of 256 samples where $d_{1\mathrm{NN}} < d_{2\mathrm{NN}}/3$. Swapping in medium region flips behavior. }
\label{tab:swap_memorization}
\small
\begin{tabular}{lcc}
\toprule
\textbf{Condition} & \textbf{Noise Range ($\sigma$)} & \textbf{Mem. Rate} \\
\midrule
 \\
EDM-1K (default sample) & $[0.002, 80]$ & 92.2\% \\
EDM-50K (default sample)  & $[0.002, 80]$ & 0.0\% \\
\midrule
\multicolumn{3}{l}{\textit{EDM-1K $\to$ EDM-50K swap:}} \\
\quad large & $\sigma > 8.4$ & 93.0\% \\
\quad \textbf{medium} & $[0.14, 8.4]$ & \textbf{0.0\%} \\
\quad small & $\sigma < 0.14$ & 91.0\% \\
\midrule
\multicolumn{3}{l}{\textit{EDM-50K $\to$ EDM-1K swap:}} \\
\quad large & $\sigma > 8.4$ & 0.0\% \\
\quad \textbf{medium} & $[0.14, 8.4]$ & \textbf{92.2\%} \\
\quad small & $\sigma < 0.14$ & 0.0\% \\
\bottomrule
\end{tabular}
\end{table}

\paragraph{Per-Noise-Level memorization.}
Motivated by the definition above, we identify a finer-grained notion of memorization at each noise level $\sigma$.
Given a training set $\mathcal D \subset \mathbb R^d$ sampled from a population distribution $p$, a denoiser $m^\bullet_\sigma$ is said to have memorized $\mathcal D$ at a noise level $\sigma$ if for (nearly) all sample test points $x_{\text{test}} \sim p$ and a large portion of noise samples $\bm{Z} \sim \mathcal N(0,I_d)$, we have $d_{1\mathrm{NN}}(m_\sigma^\bullet(x_{\text{test}} + \sigma \bm{Z})) < d_{2\mathrm{NN}}(m_\sigma^\bullet(x_{\text{test}} + \sigma \bm{Z}))/3$. Under this notion we can examine denoiser performance at a specific noise level.

With these two metrics, we now analyze where memorization arises for \textbf{EDM-1K} and \textbf{EDM-50K}. As expected and shown in \Cref{tab:swap_memorization}, \textbf{EDM-1K} exhibits strong trajectory-level memorization (92.2\% of samples), while \textbf{EDM-50K} generalizes well (0\% memorization); we omit \textbf{EMP-1K}, which is expected to exhibit 100\% memorization. We now examine per-noise-level memorization across the three models. Figure~\ref{fig:mem_memorization_rate} shows that \textbf{EDM-50K} exhibits no memorization for any $\sigma$, whereas \textbf{EMP-1K} increasingly approaches nearest-neighbor behavior as $\sigma\!\downarrow$, reaching $\approx100\%$ for $\sigma<0.3$.
In contrast, \textbf{EDM-1K} attains its highest memorization rate in an intermediate regime (peaking at $\sim26\%$ around $\sigma\in[1.1,5.3]$) and drops to nearly $0$ memorization rate at very small $\sigma$. We include qualitative visualizations of one-step denoising outputs across different noise levels in Appendix~\ref{app:onestep_showcase}, which further illustrate this phenomenon. 
In summary, while both \textbf{EDM-1K} and \textbf{EMP-1K} exhibit strong trajectory-level memorization, the diffusion model behind \textbf{EDM-1K} in fact deviates significantly from the empirical optimal \textbf{EMP-1K}. In particular, the per-noise-level memorization behaviors are quite different from mid-to-small levels of $\sigma$s.

Finally, we would like to investigate how per-noise-level memorization affects the trajectory-level memorization. We conduct the \emph{denoiser swapping experiments} (Table~\ref{tab:swap_memorization}, see Appendix~\ref{app:denoiser_swap} for details) which show that swapping the denoiser in the medium region ($\sigma \in [0.14, 8.4]$) flips memorization behavior. 
Specifically, replacing \textbf{EDM-1K} with \textbf{EDM-50K} in the medium region drops memorization from 92.2\% to 0\%, while the reverse (swapping \textbf{EDM-50K}'s denoiser with \textbf{EDM-1K}'s in this range) increases it from 0\% to 92.2\%. In contrast, swapping in large or small regions has negligible effect.

Taken together, these results reveal a qualitative mismatch between empirical optimality (\textbf{EMP-1K}) and learned memorization (\textbf{EDM-1K}): while the empirical optimal memorizes strongly across a broad noise range, the learned diffusion model memorizes primarily in an intermediate regime. This suggests that to fully understand memorization in diffusion models, we not only need to analyze the trajectory-level behavior but also the per-noise-level contributions. In particular, if one wants to mitigate memorization, it is crucial to understand and control the denoiser behavior at per-noise-level, especially in the medium noise regime.

\section{Geometric Interpretation of Noise Regimes}\label{sec:interpretation}
Section~\ref{sec:empirical_investigation} shows that per-noise-level memorization concentrates at medium noise levels and contributes most to trajectory-level memorization {(as shown by the denoiser swapping experiments in Table~\ref{tab:swap_memorization})}. In this section, we seek to provide a geometric interpretation of this phenomenon. In particular, we characterize the three regimes and explain why the two ends are ``calm'' while the middle is ``wild''. We further identify a ``danger zone'' of memorization risk in the medium noise regime and provide a practical mitigation strategy in \Cref{sec:experiments}.

\subsection{Key Players for Per-Noise-Level Memorizations}

To analyze per-noise-level memorization we examine the training loss \Cref{eq:training loss} when $\sigma$ is fixed, which reduces to $\mathrm{MSE}_\sigma(\pd)$ (cf. \Cref{eq:denoising mse}).
Two major players arise from this loss: the empirical optimal denoiser $m_\sigma$, 
which the model attempts to fit, and the distribution of noisy training samples $\bm{X}_\sigma$, which determines the supervised region. With sufficient model capacity, the trained denoiser $m_\sigma^\theta$ can match $m_\sigma$ on this region; but if this supervision is not of sufficient coverage of test samples, the trained denoiser might not generalize beyond that. Consequently, whether this induces per-noise-level memorization depends on how these two players interact as the noise level varies.

To understand these two players, we analyze two quantities below: the \emph{posterior weight}, induced from the empirical optimal denoiser, and the so-called \emph{Gaussian shell coverage},  
induced from the distribution of training samples. Their combined behavior 

at different noise levels leads to our characterization of three distinct noise regimes (cf. \Cref{section:characterization}).

\noindent \textbf{Posterior Weight.} Recall that the functions $w_i(x,\sigma)$ defined in \Cref{eq:posterior weight}
are the posterior weights 
on the data samples $x_i$ at a noise level $\sigma$. From \Cref{eq:empirical optimal denoiser} it is clear that the posterior weights determine how close the empirical optimal denoiser is to a given data point. In particular if at a given noise level one weight is always substantially larger than the rest, the empirical optimal denoiser becomes concentrated near the corresponding data points, resulting in per-noise-level memorization of the training data. As a trained model attempts to fit the empirical optimal denoiser, it is worth investigating which noise levels yield this effect. To make this more precise, we define the following quantity.

\begin{definition}
The \emph{posterior weight for the dataset $\mathcal D = \{x_i\}_{i=1}^N$} at a noise level $\sigma$ is given by the quantity

\[
W_\sigma (\mathcal D) := \mathbb E_{\substack{\bm{X}\sim \pd\\ \bm{Z}\sim \mathcal N(0,I_d)} } \max_{1\leq i\leq N}w_i(\bm{X}_\sigma,\sigma)
\]
\end{definition} 
Note that we only take expectation over the training data. This is due to the fact that during training, the model only has access to $\bm{X}_\sigma$ where $\bm{X}$ is drawn from the training data distribution.

This definition captures the expected maximum posterior weight over all noisy samples at a given noise level sigma. This quantity can be easily estimated on a given dataset; see e.g., \Cref{fig:true-vs-maxweight}, where we graph a curve of the weight using the CIFAR-10 dataset. The figure shows an interesting sharp transition window at moderate noise levels which we aim to provide a theoretical explanation next.

First of all, consider any sample data point, and w.l.o.g, assume we consider $x_1 \in \mathcal D$. Although the maximum weight $\max_{i}w_i(x_1+\sigma \bm{Z},\sigma)$ is present in the definition above, this maximum is usually attained at moderate noise levels for $i=1$, i.e., the maximum is attained at the data point itself. See \Cref{fig:max-vs-w1-weights} for an empirical validation, where the two quantities are compared across a full noise schedule. 
For this reason we provide the following theorem that gives noise thresholds for high probability weight concentration.

\begin{theorem}
\label{thm:post_weight}
Let $x_1 \in \mathcal D$ and condition on $\bm{X} = x_1$. Let $w_1(\bm{X}_\sigma,\sigma)$ denote the posterior weight on $x_1$ at a noise level $\sigma$. Then for any $\delta \in (0,1)$ and $q\in (\frac{1}{2},1)$, with probability at least $1-\delta$, we have that: 
\begin{itemize}
\item[1.] If
$\sigma \geq \min_{K>1}\frac{d_{K\mathrm{NN}}(x_1)}{a_{K-1,\delta,q}}$,        
then $w_1(\bm{X}_\sigma,\sigma) \leq q.$
\item[2.] If
$\sigma \leq \frac{d_{2\mathrm{NN}}(x_1)}{b_{\delta,q}}$,
then $w_1(\bm{X}_\sigma,\sigma) \geq q.$

\end{itemize}
where the constants $a_{K,\delta,q}$ and $b_{\delta,q}$ are defined by 
\begin{align*}
 a_{K,\delta,q} &:= F^{-1}(\delta/K) + \sqrt{\left(F^{-1}(\delta/K)\right)^2+2\log\left(\frac{Kq}{1-q}\right)}\\
 b_{\delta, q} &:= \tilde F^{-1}(\delta/N) + \sqrt{\left(\tilde F^{-1}(\delta/N)\right)^2 + 2\log\left(\frac{Nq}{1-q}\right)}
\end{align*}
and where $F$ is the CDF of the standard normal distribution and $\tilde F := 1- F$. (Note that $d_{1\mathrm{NN}}(x_1) = 0$.)
\end{theorem}
\begin{remark}

We consider a dataset consisting of 1K images randomly sampled from the CIFAR-10 dataset. Using the bounds given by \Cref{thm:post_weight}, over a random sample of 20 points, the average lower bound
needed to guarantee posterior concentration below $0.6$ with probability $0.95$ was $\sigma \geq 22.54$. The average upper bound needed to guarantee posterior concentration above $0.95$ with probability $0.95$ was $\sigma \leq 2.25$. This upper bound compares well with \Cref{fig:true-vs-maxweight}, where in the same setting $\sigma= 2.25$ lies close to where the curve of the average max weight reaches 0.95. 
\end{remark}
\begin{remark}
In Section \ref{sec:cosine-similarity} we use a variant of \Cref{thm:post_weight} to theoretically explain the sharp alignment between the optimal and conditional vector fields observed in high dimensional flow matching settings by \cite{bertrand2025closed}. 
\end{remark}

\noindent \textbf{Gaussian Shell Coverage.}\label{Gaussianshells}
\cite{song2025selective} proposes using Gaussian shells around training data points to model supervised regions during training. Following their intuition, we will use a theoretically tighter bound of Gaussian shells (cf. \Cref{lem:one_shell_lemma_main}) and propose a novel notion of data coverage utilizing the supervised regions via Gaussian shells.
The following high dimensional concentration result is key to our analysis in this paper (see proof in \Cref{sec:Gaussian shell theory}).
\begin{lemma}[Gaussian concentration \cite{laurent_massart_2000}]
\label{lem:one_shell_lemma_main}
Let $d \ge 2$ and fix $c>0$. Define
$r^{\mathrm{in}}_{c,d} := \sqrt{d - 2\sqrt{cd}}, 
$ and 
$r^{\mathrm{out}}_{c,d} := \sqrt{d + 2\sqrt{cd} + 2c}$.
     Let $\bm Z\sim\mathcal N(0,I_d)$. Then, 
$\mathbb P\big(\|\bm Z\|\in[r^{\mathrm{in}}_{c,d},r^{\mathrm{out}}_{c,d}]\big)\ge 1-2e^{-c}.$
\end{lemma}

This result states that $Z\sim \mathcal{N}(0,I_d)$ will concentrate in a thin shell around the origin with radius $\approx\sqrt{d}$ and thickness controlled by $c$. If we let $c=5$, then the concentration probability is at least 0.9865. So we will simply let $c=5$ through this paper unless otherwise stated.

Based on the concentration result, for any $x \in \mathbb{R}^d$ and $\sigma>0$, we define the
\emph{Gaussian shell} centered at $x$ by
\begin{equation}\label{eq:Gaussian shell}
    \sh{\sigma}{x}
:= \Bigl\{ x' \in \mathbb{R}^d : 
\sigma r^{\mathrm{in}}_{c,d} \le \|x' - x\| \le \sigma r^{\mathrm{out}}_{c,d} \Bigr\}.
\end{equation}

 As a direct consequence of \Cref{lem:one_shell_lemma_main}, we have that 
\[\mathbb P\big(\bm{X}_\sigma\in \sh{\sigma}{x}\mid \bm{X}=x\big)\ge\ 1-2e^{-c}.\]
This shows that a single Gaussian shell provides a high-probability description of where a noisy data point lies. Based on the notion of Gaussian shells, we next introduce the notion of data coverage.

\begin{definition}[Gaussian Shell Coverage]\label{def:data-coverage}
Suppose $\mathcal D= \{x_i\}_{i=1}^N$ is a finite dataset sampled from the data distribution $p$. We define the \emph{Gaussian shell coverage} of the distribution $p$ relative to $\mathcal D$ at the noise level $\sigma$ as follows:
\[
C_\sigma(p,\mathcal{D}) :=\mathbb P\left(\bm{X}_\sigma \in \bigcup_{i=1}^N \sh{\sigma}{x_i}\right), \quad\text{where }\bm{X}\sim p
\]
\end{definition}
The quantity approximately quantifies the likelihood that a diffusion model encounters a noisy test $x_\text{test} + \sigma z$ during training. In other words, $C_\sigma$ measures how large a portion of the underlying ground-truth distribution is covered by the training set after the convolution $p_\sigma=p*\mathcal{N}(0,\sigma^2I)$. For brevity, in the remainder of the paper we will refer to the Gaussian shell coverage $C_\sigma(p,\mathcal{D})$ simply as ``coverage” when no confusion can arise.
Note that while in applications the data distribution $p$ is not known, this quantity can be empirically estimated if a test dataset is available; see \Cref{fig:true-vs-maxweight} where we estimate $C_\sigma$ for the CIFAR-10 dataset.

\Cref{fig:true-vs-maxweight} also exhibits an interesting sharp phase transition for the coverage as the noise level varies. 
To provide a theoretical understanding why this happens, we find it better to treat the dataset as random and study the coverage in expectation.
Interestingly, it turns out that controlling this coverage is closely related to the probability mass of intersection between just two shells. Below we provide a clean description of such a probability mass.
\begin{definition}
    Let $e_1=[1,0,\cdots,0]^T\in\R^d$ and define
\[
\Phi_{d,c}\left(t\right)
:=
\mathbb P\!\left(
\|\bm Z\|,\left\|\bm Z+t e_1\right\|\in[r^{\mathrm{in}}_{c,d},r^{\mathrm{out}}_{c,d}]
\right).
\]
\end{definition}
The function $\Phi_{d,c}(t)$ measures the probability of a Gaussian noisy observation for one data point lying also in the Gaussian shell over a data point at distance $t$ and hence roughly represents the ``probability of seeing a noisy data in the intersection of two Gaussian shells".
The following result shows that the coverage is closely controlled by $\Phi_{d,c}$. 

\begin{theorem}
\label{thm:shell-coverage}
Let $\bm{X}, \bm{X}_1, \dots, \bm{X}_N \overset{\text{i.i.d.}}{\sim} p$, and let $\bm{Z} \sim \mathcal N(0,I_d)$ be independent of $(\bm{X},\bm{X}_1,\dots,\bm{X}_N)$. Fix a noise level $\sigma>0$. Define the union of shells
$\mathcal U_\sigma := \bigcup_{i=1}^N \sh{\sigma}{\bm{X}_i}.$ Then the coverage probability satisfies
\[
\begin{aligned}
\mathbb P(\bm{X}_\sigma \in \mathcal U_\sigma)
&\ge
\mathbb E\!\left[
\Phi_{d,c}\!\left(\frac{d_{1\mathrm{NN}}(\bm X)}{\sigma}\right)
\right],\\
\mathbb P(\bm{X}_\sigma \in \mathcal U_\sigma)
&\le
2e^{-c}
+
N\,\mathbb E\!\left[
\Phi_{d,c}\!\left(\frac{\|\bm{X}-\bm{X}'\|}{\sigma}\right)
\right]
\end{aligned}
\]
where  $\bm{X}'$ is an independent copy of $\bm{X}$.
\end{theorem}
The two bounds in the theorem are governed by the distribution of pairwise distances in the data as well as the ambient dimension $d$ reflected by $\Phi_{d,c}$. The lower bound actually provides a very accurate estimation empirically and one may directly use it as a proxy to estimate the coverage; see \Cref{fig:shell-coverage-cifar10} for experimental results on CIFAR-10.

\subsection{Characterization of the Noise Regimes}\label{section:characterization}

With the two quantities of posterior weight and data coverage defined and explored, we are ready to characterize the noise regimes. 
In \Cref{fig:true-vs-maxweight} we plot two curves of the two quantities against noise levels, based on a 1k subset of the CIFAR-10 dataset. 

Note that in \Cref{fig:mem_memorization_rate}, 
\textbf{EDM-1K} exhibits a positive memorization ratio for noise levels range $\sigma\in[0.6,12]$, which closely aligns with the transition interval observed in \Cref{fig:true-vs-maxweight}. 
Based on all these observations, we now informally identify the three noise regimes as follows using the two quantities:

\begin{figure}[t]
    \centering
    \includegraphics[width=0.85\linewidth]{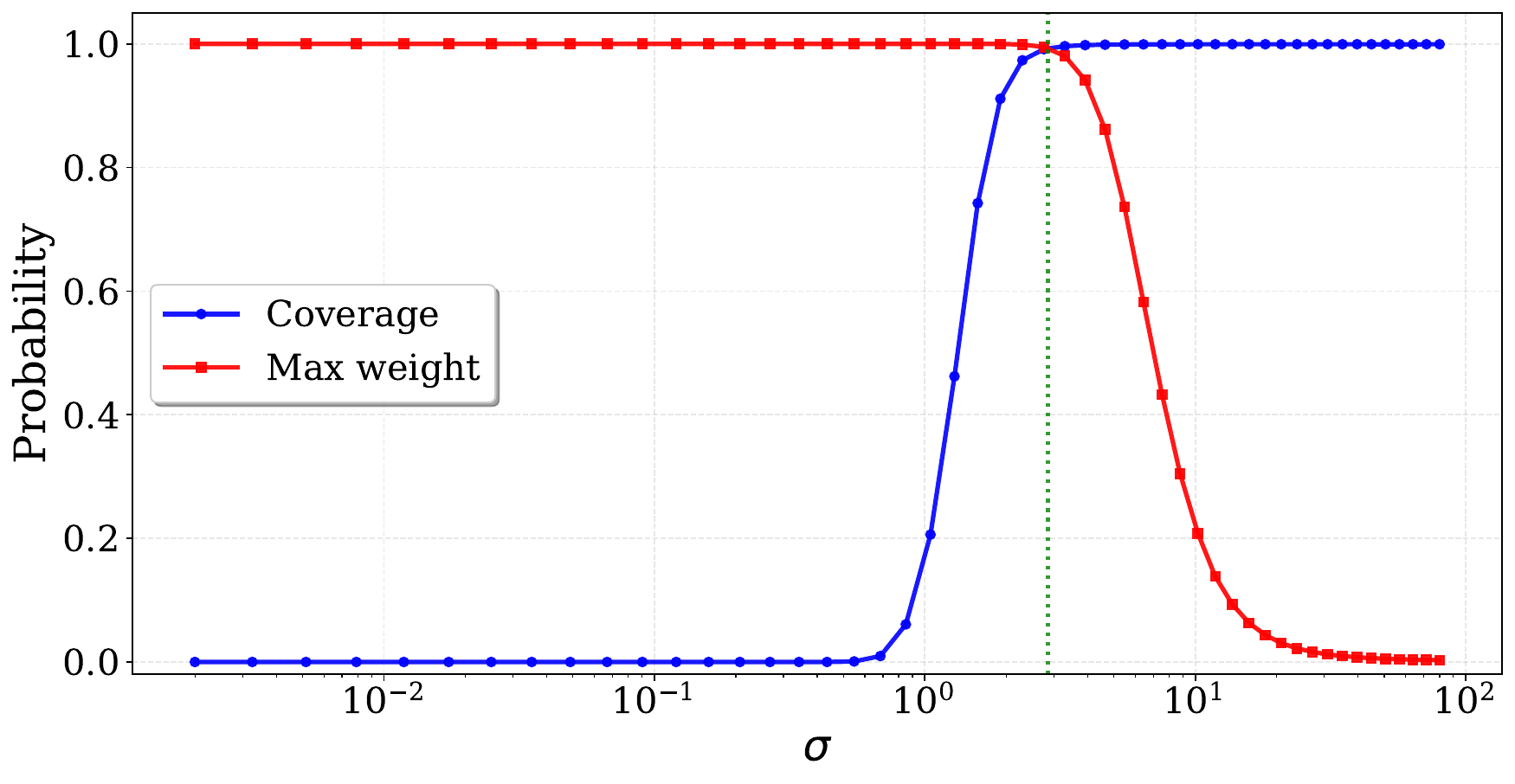}
    \caption{\textbf{Gaussian shell coverage and max posterior weight on CIFAR-10.}
We plot empirical estimates of the Gaussian shell coverage $C_\sigma(p,\mathcal D)$ and the max posterior weight $W_\sigma(\mathcal D)$ as functions of the noise level $\sigma$, using a $1\mathrm{k}$ CIFAR-10 training subset and $1\mathrm{k}$ held-out test images. }\vspace{-0.1in}
    \label{fig:true-vs-maxweight}
\end{figure}

\begin{center}
\small
\setlength{\tabcolsep}{4pt}
\begin{tabular}{lcc}
\toprule
Regime & Posterior weight & Shell coverage \\
\midrule
Small noise  & High            & Low  \\
Medium noise & Phase transition & Phase transition \\
\ \ -Danger zone-  & High            & High \\
Large noise  & Low             & High \\
\bottomrule
\end{tabular}
\end{center}

While informal this description may serve as a guiding principal for classifying training dynamics across a full noise schedule. Precise estimation of noise ranges requires further work, but our \Cref{thm:post_weight} and \Cref{thm:shell-coverage} serve as a starting point. In what follows, we provide rationales for this characterization and further analysis within each regime.

\subsubsection{Small Noise Regime}\label{sec:small sigma}

The small noise (small $\sigma$) regime is characterized by high posterior weight and low coverage.
This is where the empirical optimal denoiser behaves almost like a nearest neighbor map. If a trained model well approximates the empirical optimal one, it will absolutely produce memorized training data and hence it is natural to suggest mitigation for memorization in this region~\cite{wan2025elucidating,baptista2025memorization}. However, although the empirical optimal denoiser exhibits strong memorization behavior (which is almost locally constant), the trained denoiser is unlikely to fit the empirical optimal denoiser since the supervision region during training, described by shell coverage, is very limited.

We have empirically observed in Figure~\ref{fig:mem_memorization_rate} that the small noise regime is quite safe from memorization risk in per-noise-level: while the empirical optimal denoiser exhibits nearly 100\% memorization for $\sigma < 0.3$, the learned model (\textbf{EDM-1K}) shows minimal memorization in this regime. We also validate in the swapping experiment in \Cref{tab:swap_memorization} that the small noise regime in a trained diffusion model does not contribute to memorization. 

Within the small noise regime, when $\sigma$ is small enough (e.g., $<0.2$ for the CIFAR-10), all Gaussian shells around data points will be disjoint. In this setting, we have the following description of the optimal solution to the training objective. See Theorem~\ref{thm:inf-many-minimizers-disjoint-shell} for the complete statement and proof.

\begin{theorem}[Informal theorem]
\label{thm:inf-many-minimizers-disjoint-shell-informal}
Suppose the training is only restricted to the union of shells $\bigcup_{i=1}^N \sh{\sigma}{x_i}$, which are also pairwise disjoint for small $\sigma$. Then, there are infinitely many global minimizers of the denoising objective sending $\sh{\sigma}{x_i}$ to $x_i$ for $i=1,\ldots,N$.
\end{theorem}

In the small-$\sigma$ regime, training supervision is concentrated within the small region consisting of the union of Gaussian shells. Then by Theorem~\ref{thm:inf-many-minimizers-disjoint-shell-informal}, the training objective alone does not constrain the denoiser outside this small region, so extrapolation behavior is largely determined by the model’s inductive bias.
In one dimension, shallow network denoisers have been shown to interpolate the data with piecewise linear behavior \citep{zeno2023minimum}.
In high dimensional and deep network settings, this behavior may not hold. Below we try to provide some evidence that certain locality inductive bias will appear for image datasets.

For an image data such as CIFAR-10, we hypothesize that the locality increases when $\sigma$ decreases: each output pixel depends mainly on a local patch. The task therefore reduces to patch denoising in a much lower-dimensional subspace, where the effective training data, the extracted patches, are plentiful, making the problem easier. This aligns with empirical findings in~\citep{kamb2025an,lukoianov2025locality,niedoba2024towards}: trained neural denoisers exhibit shrinking receptive fields as the noise level decreases.
\cite{kamb2025an,lukoianov2025locality,wang2025seeds} finds the optimal denoiser under locality bias and \cite{lukoianov2025locality} connects the localization to the data statistics. 
However, it was not clear what affects the spatial decay rate of the sensitivity. To study this analytically, we model the data as a Gaussian with circulant covariance---a natural idealization, since natural images exhibit approximately translation-invariant second-order statistics and stationary Gaussian models are a standard tool in their analysis~\citep{field1987relations,simoncelli2001natural}. We propose the following theorem; formal statements and proofs are in Appendix~\ref{app:sensitivity-localization}.

\begin{theorem}[Informal theorem]
For a cyclic stationary Gaussian data distribution, at any noise level $\sigma$, the sensitivity between two pixels is bounded by $\frac{T_\sigma}{m}$, where $m$ is their wrap-around distance and $T_\sigma$ is a constant depending only on the data spectrum and $\sigma$. A more concentrated spectrum leads to faster decay of the sensitivity.
\end{theorem}

\subsubsection{Large Noise Regime}\label{early_stage}
We characterize the large noise regime by low posterior weight and high coverage. The latter implies that the model $m_\sigma^\theta$ can learn the empirical optimal on a region that covers most of the mass of $p_\sigma$, where inference takes place. In this regime, the empirical optimal denoiser does not exhibit per-noise-level memorization behavior, and hence even if the trained denoiser closely approximates the empirical optimal denoiser, per-noise-level memorization risk remains low. 

In fact, we can provide some more refined description of the empirical optimal denoiser within this regime.

It is known that when $\sigma$ is large, the denoiser $m_\sigma$ will be close to the mean $\mu$ of the data distribution $p$ \cite{wan2025elucidating}. Empirically, it has also been observed that $m_\sigma$ behaves like a linear map when $\sigma$ is large \cite{li2024understanding}. In what follows, we reconcile these two observations by establishing the limiting behavior of $m_\sigma$ when $\sigma$ is large.

Let $\Sigma$ denote the covariance of $p$, {and recall that $\mu$ is the mean of $p$.} Consider now the Gaussian distribution $\mathcal{N}(\mu,\Sigma)$ {that matches the mean and covariance of $p$. The denoiser $m_\sigma^G$ of this Gaussian distribution} is linear and has the closed form \cite{li2024understanding}:
\begin{equation}\label{eq:Gaussian closed form}
    m_\sigma^G(x)=\mu+\Sigma(\Sigma + \sigma^2 I)^{-1}(x-\mu).
\end{equation}
Note that when $\sigma$ is large, the linear term will be small and hence $m_\sigma^G\approx\mu$. 
It turns out that this optimal linear denoiser for Gaussian captures the limiting behavior of $m_\sigma$ when $\sigma$ is large (see \Cref{sec:large noise taylor} for a proof).

\begin{theorem}\label{thm:gauss-plus-excess}
Assume $\supp(p)\subset B_R(0)$, where $B_R(0)$ is the open ball of radius $R$ centered at the origin.
Then, for any $\delta\in(0,1)$ there exists a bounded continuous function
$H$ on $\overline{B_R(0)}\times[-\delta,\delta]$ such that
\begin{equation}\label{eq:gauss-plus-excess-sigma}
m_\sigma(x)
=
m_\sigma^G(x)
+H(x,\sigma)(1+\sigma)^{-2}.
\end{equation}
\end{theorem}

This theorem indicates that for large $\sigma$, the denoiser (the one either for the true data distribution $p$ or for the empirical distribution over a dataset) is close to being a linear map determined by the mean and covariance of $p$. This validates the empirical observations in \cite{li2024understanding}.

\begin{figure}[t]
    \centering
    \includegraphics[width=0.95\linewidth]{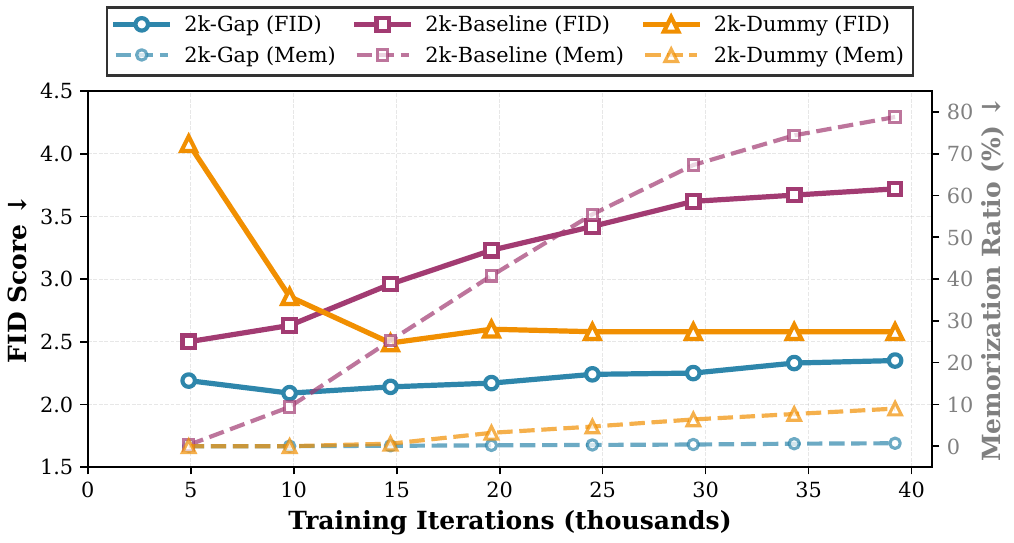}
    \caption{\textbf{Comparison of anti-memorization methods.} Training curves showing FID score (solid lines) and memorization ratio (dashed lines) for three methods: 2k-Baseline, 2k-Dummy, and 2k-Gap. Gap training achieves the lowest memorization (0.7\%) and best FID (2.35) at final checkpoint.}\vspace{-0.1in}
    \label{fig:gap-training}
\end{figure}

\subsubsection{Medium Noise Regime}

We characterize the medium noise  regime by simultaneous transitions in posterior weight concentration and data coverage: as sigma decreases from large to small, posterior weights concentrate on nearest neighbors while data coverage becomes incomplete. The interplay between these two effects creates a volatile training environment.

There are two key observations we have regarding the medium noise regime. The first is the discovery of a \emph{danger zone} for memorization: note that in \Cref{fig:true-vs-maxweight} there is a small region of $\sigma$ in the medium regime where both coverage and posterior weights are high. In this case, the denoiser $m_\sigma^\theta$ will tend to learn the empirical optimal $m_\sigma$ in a region of $\R^d$ where $p_\sigma$ is concentrated {(due to high Gaussian shell coverage)} and furthermore, since this empirical optimal $m_\sigma$ exhibits high per-noise-level memorization due to the high posterior concentration, $m_\sigma^\theta$ will likely exhibit high per-noise-level memorization. 

The other observation is that generalization appears to arise primarily in the medium sigma regime. Although we do not yet have a theoretical explanation for this behavior, our empirical results provide strong evidence: swapping denoisers in this regime flips memorization behavior (Table~\ref{tab:swap_memorization}), with qualitative examples shown in \Cref{fig:mem_sample_grid}.
 This highlights the medium region as the key regime governing both memorization and generalization in diffusion models.

More theoretical understanding of this regime is left to future work. In the section below, we provide interesting empirical findings regarding this regime: since the medium sigma regime is the most dangerous for memorization, we conjecture that avoiding it during training can mitigate memorization. We validate this idea experimentally in Section~\ref{sec:experiments}.

\begin{figure}[t]
\centering
\includegraphics[width=0.85\textwidth]{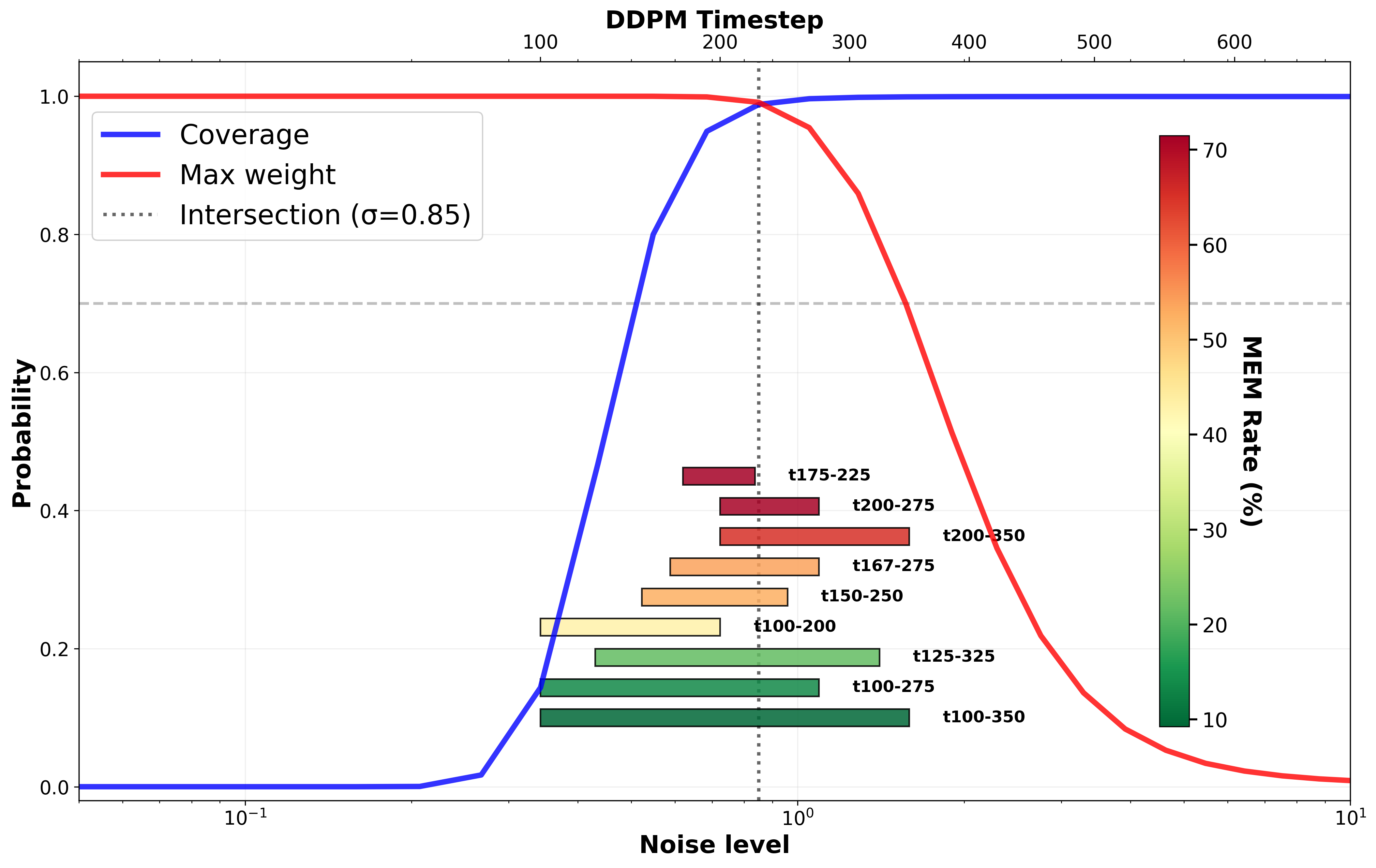}
\caption{
\textbf{Danger zone and gap configurations.} The intersection of coverage and max weight curves identifies the danger zone. Memorization rates for different gap configurations are overlaid, showing
  that gaps targeting this region minimize memorization.}\vspace{-0.1in}
\label{fig:danger-zone-combined}
\end{figure}

\section{Targeted Noise Region Undertraining}
\label{sec:experiments}

Our theoretical analysis has identified a ``danger zone'' in the medium-$\sigma$ regime where memorization risk is highest due to the confluence of high posterior concentration and significant data coverage. A natural question arises: can we mitigate memorization by targeting this specific regime during training? While the swap experiments in \Cref{tab:swap_memorization} already provides some positive answer, in this section, we further test a rejection training idea on CIFAR-10 and CelebA.

\noindent \textbf{CIFAR-10 car subset.}
We consider a CIFAR-10 car subset (2k images) where both generalization and memorization occur~\cite{yoon2023diffusion}. As a baseline, we train standard EDM~\citep{karras2022elucidating} with the full noise schedule $\sigma \in [0.002, 80]$  (\textbf{2k-Baseline}). Building on the baseline, we consider two anti-memorization methods: \textbf{2k-Dummy}~\citep{yoon2023diffusion} adds 2k random noise images as a second class to provide distributional contrast; and \textbf{2k-Gap} (ours) excludes the danger zone $\sigma \in [1.0, 5.0]$, identified empirically via coverage and posterior weight analysis (Appendix~\ref{app:danger_zone}). We train for 40k iterations, and evaluate at 8 checkpoints. At each checkpoint, we generate 1,024 samples and compute FID against 5k CIFAR-10 training cars and the memorization ratio.

Figure~\ref{fig:gap-training} shows training trajectories. The 2k-Baseline exhibits severe memorization (78.8\%) with FID 3.72. The 2k-Dummy method substantially reduces memorization to 9.0\% with FID 2.58. Our 2k-Gap achieves the best results: 0.7\% memorization and FID 2.35.

\paragraph{Grayscale CelebA training data.}
Following \cite{bonnaire2025why}, we use downscaled grayscale CelebA data with 1024 training samples, which diffusion models can memorize. We compute the coverage and max weight curves (Figure~\ref{fig:danger-zone-combined}) to identify the danger zone around $\sigma \in [0.55, 1.3]$ (approximately DDPM timesteps 125--325). We test multiple timestep gap configurations around this region to study the effect on memorization.
We utilize the same DDPM architecture as in~\cite{bonnaire2025why} and train for 100k iterations with batch size 512 and the baseline achieves 74.97\% memorization.
We then test nine timestep gap configurations, removing different ranges within [100--350] (the danger zone plus some buffer) during training.

Results are summarized in Figure~\ref{fig:danger-zone-combined}. Skipping the full range [100--350] significantly reduces memorization from 75\% to 9\%. High max weight regions are more influential: the gap [100--200] spanning only 100 timesteps reduces memorization to 43\%, while [100--275] achieves 13\%.
Visual comparisons of generated samples and their nearest training neighbors are provided in Appendix~\ref{app:pareto}, showing high sample quality and meaningful generalization, which validates the memorization mitigation effect. We also include in~Appendix~\ref{app:pareto} a Pareto frontier analysis over all gap configurations, showing the trade-off between image quality (FID) and memorization; mid-range gaps (150--250, 167--275) achieve the best balance.

\section{Discussion}

Our results indicate that the trajectory-level memorization in diffusion models arises from a specific medium range of noise levels, rather than from the small or large noise extremes. This per-noise-level analysis provides a fundamentally different perspective for analyzing memorization and generalization.

From a practical perspective, this points to targeted control strategies that act selectively on the intermediate noise regime, such as modifying training emphasis, scheduling, or supervision strength at those scales. More broadly, our framework highlights the importance of incorporating data geometry and noise-scale interactions into the design and analysis of diffusion models.

\section*{Acknowledgements}
This work is partially supported by the National Science Foundation (NSF) under grants CCF-2112665, MFAI-2502083, and MFAI-2502084, and by the Defense Advanced Research Projects Agency (DARPA) under Contract No. HR001125CE020.

\bibliographystyle{plain}
\bibliography{cite}
\newpage
\appendix

\startcontents[appendix]
\section*{Appendix Contents}
\printcontents[appendix]{}{1}{\setcounter{tocdepth}{2}}
\vspace{1em}

\section{Posterior Concentration and Cosine Similarity Concentration}

\subsection{Proof of \Cref{thm:post_weight}}\label{sec:concentration_thm}
\begin{proof}
Fix $\delta \in (0,1)$ and $q \in (\frac{1}{2},1)$. Recall that conditioning on $\bm X = x_1$ gives $\bm X_\sigma :=x_1+\sigma \bm Z$ where $\bm Z\sim N(0,I_d)$. By definition
\[
w_j(\bm X_\sigma,\sigma) = \frac{\exp\left(  -\frac{\|\bm X_\sigma - x_j\|^2}{2\sigma^2}\right)}{\sum_{k=1}^N\exp\left(-\frac{\|\bm X_\sigma - x_k\|^2}{2\sigma^2}\right)}  . 
\]

For ease of notation write $w_j = w_j(\bm X_\sigma,\sigma)$ and notice that since $\sum_{j=1}^Nw_j = 1$ we have
\begin{equation}\label{eq:weight_ratio}
    w_1 = \dfrac{1}{1+\sum_{j\neq 1}\frac{w_j}{w_1}}.
\end{equation}
Let $\hat d_j:= x_j-x_1$ for all $x_j \in \mathcal D$. By direct calculation we have 
\[
\frac{w_j}{w_1} = \exp\Big(\frac{1}{\sigma}\langle \bm Z, \hat d_j\rangle -\frac{1}{2\sigma^2}\|\hat d_j\|^2\Big).
\]
Define 
\[
\bm V_j : = \frac{1}{\sigma}\langle \bm Z, \hat d_j\rangle -\frac{1}{2\sigma^2}\|\hat d_j\|^2
\]
and without loss suppose that for all $K \in \{1,\dots,N-1\}$, we have $\|x_{K+1}-x_1\| = d_K$ where $d_K := d_{(K+1)NN}(x_1)$ is the distance from $x_1$ to its $K^{\text{th}}$ nearest neighbor in $\mathcal D$. (Under this convention $d_{1NN}(x_1) = 0$).
Fix a $K \in \{2,\dots,N\}$ and let
\[
c_{K} := \log \left(\frac{1-q}{qK}\right).
\]
Note $c_K <0$ since $q \in (\frac{1}{2},1)$ and therefore $(1-q)/q \in (0,1)$.
Define the event 
\[
\mathcal E_K := \bigcap_{j=2}^{K+1}\Big\{\bm V_j \geq c_K\Big\}
\]
so that on $\mathcal E_K$,
\[
\sum_{j\neq 1} \frac{w_j}{w_1} = \sum_{j\neq 1}e^{\bm V_j} \geq \sum_{j=2}^{K+1}e^{\bm V_j} \geq Ke^{c_K} = \frac{1-q}{q}
\]
which gives $w_1 \leq q$ via \Cref{eq:weight_ratio} . We now bound the complement event $\mathcal E_K^c$. By union bound we have 
\begin{equation}\label{eq:early_union}
    \mathbb P(\mathcal E_K^c) \leq \sum_{j=2}^{K+1}\mathbb P(\bm V_j < c_K).
\end{equation}
For any fixed $j \neq 1$, the key idea is to observe that $\langle \bm Z, \hat d_j\rangle \sim N(0,\|\hat d_j\|^2)$. Hence 
\begin{equation}\label{eq:early_Vj}
    \mathbb P(\bm V_j < c_K) = \mathbb P\Big(\langle \bm Z,\frac{\hat d_j}{\|\hat d_j\|}\rangle < \frac{\|\hat d_j\|}{2\sigma}+\frac{\sigma c_K}{\|\hat d_j\|}\Big)  = F\Big(\frac{\|\hat d_j\|}{2\sigma}+\frac{\sigma c_K}{\|\hat d_j\|}\Big).
\end{equation}
where $F$ is the CDF of the standard normal distribution. Put $z:= F^{-1}(\delta/K)$ and note that since $F$ and $F^{-1}$ are increasing functions we may write
\begin{equation}\label{eq:early_iff}
    \frac{\|\hat d_j\|}{2\sigma}+\frac{\sigma c_K}{\|\hat d_j\|} \leq z \quad \iff \quad F\Big(\frac{\|\hat d_j\|}{2\sigma}+\frac{\sigma c_K}{\|\hat d_j\|}\Big) \leq \frac{\delta}{K}.
\end{equation}
Solving the equation 
\[
\frac{a}{2}+\frac{c_K}{a} = z
\]
yields a positive solution 
\[
a_{K,\delta,q} := z+\sqrt{z^2-2c_K}
\]
which is well defined since $c_K <0$. Calculus shows that the function $g(a) = \frac{a}{2} + \frac{c_K}{a}$ is increasing for all $a>0$, hence for all $j$,
\begin{equation}\label{eq:early_imply}
    \frac{\|\hat d_j\|}{\sigma} \leq a_{K,\delta,q} \quad \implies \quad \frac{\|\hat d_j\|}{2\sigma}+\frac{\sigma c_K}{\|\hat d_j\|} \leq z.
\end{equation}
Thus by \Cref{eq:early_iff} and \Cref{eq:early_imply} we have the following implication:\\
If
\[
\sigma \geq \max_{2\leq j\leq K+1}\frac{\|\hat d_j\|}{a_{K,\delta,q}} = \frac{d_{K}}{a_{K,\delta,q}}
\]
then
\[
F\Big(\frac{\|\hat d_j\|}{2\sigma}+\frac{\sigma c_K}{\|\hat d_j\|}\Big) \leq \frac{\delta}{K}
\]
for all $j \in \{2,\dots,K+1\}$. By \Cref{eq:early_union} and \Cref{eq:early_Vj} this further implies 
\[
\mathbb P(\mathcal E_K^c) \leq K\cdot\frac{\delta}{K} = \delta.
\]
Since this construction holds for all $K$ we have 
\[
\sigma \geq \min_{K} \frac{d_{K}}{a_{K,\delta,q}} = \min_{K>1}\frac{d_{KNN}(x_1)}{a_{K-1,\delta,q}}
\]
implies for $K_* := \arg \min_K (d_{K}/a_{K,\delta,q})$,
\[
\mathbb P(\mathcal E_{K_*}) \geq 1-\delta
\]
and on $\mathcal E_{K_*}$, 
\[
w_1 \leq q.
\]
Plugging in the respective definitions for $c_K$ and $z$ into $a_{K,\delta,q}$ completes the proof of the first implication. \\
\\
To prove the second implication, let 
\[
c := \log\left(\frac{1-q}{q(N-1)}\right)
\]
and notice that $c<0$. Define the event
\[
\mathcal F := \bigcap_{j \neq 1}\Big\{\bm V_j \leq c\Big\}
\]
so that on $\mathcal F$, 
\[
\sum_{j \neq 1} \frac{w_j}{w_1} = \sum_{j \neq 1} e^{\bm V_j} \leq (N-1)e^c = \frac{1-q}{q}
\]
which gives $w_1 \geq q$ via \Cref{eq:weight_ratio}. By union bound we have 
\begin{equation} \label{eq:late_union}
    \mathbb P(\mathcal F^c) \leq \sum_{j \neq 1}\mathbb P(\bm V_j >c).
\end{equation}
For any fixed $j \neq 1$, observe since $\langle \bm Z, \hat d_j\rangle \sim N(0,\|\hat d_j\|^2)$, 
\begin{equation} \label{eq:late_prob}
    \mathbb P(\bm V_j >c) = \mathbb P\left( \left\langle\bm  Z, \frac{\hat d_j}{\|\hat d_j\|}\right\rangle > \frac{\|\hat d_j\|}{2\sigma}+\frac{\sigma c}{\|\hat d_j\|}\right) = \tilde F\left(\frac{\|\hat d_j\|}{2\sigma}+\frac{\sigma c}{\|\hat d_j\|}\right)
\end{equation}
where $\tilde F := 1-F$. Put $\tilde z := \tilde F^{-1}(\delta/(N-1))$ and note that since both $\tilde F$ and $\tilde F^{-1}$ are decreasing functions we may write
\begin{equation}\label{eq:late_iff}
    \frac{\|\hat d_j\|}{2\sigma}+\frac{\sigma c}{\|\hat d_j\|} \geq \tilde z \iff \tilde F\Big(\frac{\|\hat d_j\|}{2\sigma}+\frac{\sigma c}{\|\hat d_j\|}\Big) \leq \frac{\delta}{N-1}.
\end{equation}
Then solving the equation
\[
\frac{b}{2}+\frac{c}{b} = \tilde z
\]
yields a positive solution 
\[
b'_{\delta,q} := \tilde z+\sqrt{\tilde z^2-2c}
\]
which is well defined since $c <0$. By a similar argument to that used in the proof of the first implication, we have for all $j \neq 1$,
\begin{equation}\label{eq:late_implies}
    \frac{\|\hat d_j\|}{\sigma} \geq b'_{\delta,q} \quad \implies \quad \frac{\|\hat d_j\|}{2\sigma}+\frac{\sigma c}{\|\hat d_j\|} \geq\tilde z.
\end{equation}
Since $d_{1}\leq \|\hat d_j\|$ for all $j$, by \Cref{eq:late_iff}  and \Cref{eq:late_implies} we have the following implication:\\
If 
\[
\sigma \leq \frac{d_{1}}{b'_{\delta,q}}
\]
then 
\[
\tilde F \Big(\frac{\|\hat d_j\|}{2\sigma}+\frac{\sigma c}{\|\hat d_j\|}\Big) \leq \frac{\delta}{N-1}
\]
for all $j \neq 1$. By the above union bound in \Cref{eq:late_union} and \Cref{eq:late_prob} this further implies 
\[
\mathbb P( \mathcal F^c ) \leq \delta.
\]
Lastly notice that since $\tilde F^{-1}$ is a decreasing function we have 
\begin{align*}
    b'_{\delta,q} &= \tilde F^{-1}(\delta/(N-1)) + \sqrt{\Big(\tilde F^{-1}(\delta/(N-1))\Big)^2 + 2\log\Big(\frac{q(N-1)}{1-q}\Big)}\\
    &\leq \tilde F^{-1}(\delta/N) + \sqrt{\Big(\tilde F^{-1}(\delta/N)\Big)^2 + 2\log\Big(\frac{qN}{1-q}\Big)} = : b_{\delta,q}
\end{align*}
which implies that $\sigma \leq \frac{d_1}{b_{\delta,q}} = \frac{d_{2NN}(x_1)}{b_{\delta,q}}$  yields the same implication. This completes the proof. 
\end{proof}

\subsection{Cosine Similarity Concentration}\label{sec:cosine-similarity}
The following lemma is similar to \Cref{thm:post_weight} with an additional bound on the distance from the denoiser $m_\sigma(\bm X_\sigma)$ to $x_1$ given that $\bm X_\sigma = x_1 + \sigma \bm Z$. This demonstrates that for small $\sigma$ the empirical optimal denoiser admits nearest neighbor behavior. Furthermore this result leads to \Cref{cor:cos_sim}, where following inspiration from \cite{bertrand2025closed}, we demonstrate a conditional bound on the cosine similarity between the optimal and conditional vector fields in the flow matching setting with linear schedules. 
\begin{lemma}\label{lem:post_weight_epsilon}
Let $\mathcal D = \{x_1, \dots,x_N\} \subset \mathbb R^d$ be a finite data set with uniform distribution $\pd$. Define 
\[
D := \max_{1\leq i,j \leq N}\|x_i-x_j\|.
\]
let $ \hat d_j := x_j-x_1$.
Fix $\epsilon>0$ and condition on $\bm X= x_1$, then with probability at least $1-\delta$, 
\[
w_1(\bm X_\sigma , \sigma) \geq 1-\epsilon
\]
and consequently
\begin{equation}
    \|m_\sigma(\bm X_\sigma) - x_1\| \leq D\left(1-w_1(\bm X_\sigma,\sigma)\right) \leq D\epsilon
\end{equation}
where 
\[
\delta := \sum_{j\neq 1}\tilde F \left(\frac{\|\hat d_j\|}{2\sigma}+\frac{\sigma \kappa_\epsilon}{\|\hat d_j\|} \right)
\]
and 
\[
\kappa_\epsilon := \log\left(\frac{\epsilon}{(N-1)(1-\epsilon)}\right)
\]
and $\tilde F := 1-F$ where $F$ is the CDF of the standard normal distribution. 
\end{lemma}
\begin{proof}
Let $\epsilon >0$ and recall that $\bm X = x_1$ gives $\bm X_\sigma = x_1 + \sigma \bm Z$. We have 
\[
w_j(\bm X_\sigma,\sigma) = \frac{\exp\left(  -\frac{\|\bm X_\sigma - x_j\|^2}{2\sigma^2}\right)}{\sum_{k=1}^N\exp\left(-\frac{\|\bm X_\sigma - x_k\|^2}{2\sigma^2}\right)} .
\]
For ease of notation write $w_j := w_j(\bm X_\sigma,\sigma)$. Since $\sum_{j=1}^Nw_j =1$ we have 
\begin{equation}\label{eq:weight_ratio_lemma}
    w_1 = \dfrac{1}{1+\sum_{j\neq 1}\frac{w_j}{w_1}}.
\end{equation}
By direct calculation 
\[
\frac{w_j}{w_1} = \exp\Big(\frac{1}{\sigma}\langle \bm Z, \hat d_j\rangle -\frac{1}{2\sigma^2}\| \hat d_j\|^2\Big).
\]
Define 
\[
\bm V_j := \frac{1}{\sigma}\langle \bm Z, \hat d_j\rangle -\frac{1}{2\sigma^2}\|\hat d_j\|^2
\]
and put 
\[
\kappa_\epsilon := \log\left(\frac{\epsilon}{(N-1)(1-\epsilon)}\right).
\]
Consider the event 
\[
\mathcal E:= \bigcap_{j\neq 1}\left\{\bm V_j \leq \kappa_\epsilon \right\}
\]
so that on $\mathcal E$, 
\[
\sum_{j\neq 1} \frac{w_j}{w_1} = \sum_{j\neq 1} e^{\bm V_j} \le (N-1)e^{\kappa_\epsilon} = \frac{\epsilon}{1-\epsilon}
\]
which gives $w_1 \geq 1-\epsilon$ by \Cref{eq:weight_ratio_lemma}. Since $\left\langle \bm Z, \hat d_j\right\rangle \sim \mathcal N(0,\|\hat d_j\|^2)$ we have by union bound 
\[
\mathbb P(\mathcal E^c) \leq \sum_{j\neq 1}\mathbb P\left(\left\langle\bm Z, \frac{\hat d_j}{\|\hat d_j\|}\right\rangle > \frac{\|\hat d_j\|}{2\sigma} + \frac{\sigma \kappa_\epsilon}{\|\hat d_j\|}\right)= \sum_{j\neq 1}\tilde F \left(\frac{\|\hat d_j\|}{2\sigma} + \frac{\sigma \kappa_\epsilon}{\|\hat d_j\|} \right) =: \delta.
\]
Recall 
\[
m_\sigma(\bm X_\sigma) = \sum_{j=1}^N w_j(\bm X_\sigma , \sigma)x_j
\]
so 
\[
\|m_\sigma(\bm X_\sigma) -x_1\| =  \Big\| \sum_{j=1}^N w_j(x_j - x_1) \Big\| \leq \sum_{j\neq 1} w_j\|x_j - x_1\| \leq D(1-w_1). 
\]
Thus on $\mathcal E$,
\[
\|m_\sigma(\bm X_\sigma) -x_1\| \leq D\epsilon
\]
which completes the proof. 
\end{proof}

For the following setting let $\bm X_t:= (1-t)\bm Z + t\bm X$ where $\bm Z \sim \mathcal N(0, I_d)$ and $\bm X \sim \pd$. Let $u_t$ and $u_t(\cdot \mid x_1)$ be the optimal and conditional vector fields respectively in the flow matching setting \cite{lipman2023flow} with scheduling functions $\alpha_t = t$ and $\beta_t = 1-t$. Then conditional vector field is given by
\begin{equation}\label{eq:cond-vec}
    u_t\left(x \mid x_1 \right) = \frac{x_1-x}{1-t}.
\end{equation}
while the empirical optimal vector field admits a closed form 
\begin{equation}\label{eq:opt-vect}
u_t(x) = \sum_{i=1}^N w_i(x,t)\frac{x_i-x}{1-t} 
\end{equation}
where the weights $w_i(x,t)$ are given by softmax weights
\[
w_i(x,t):= \frac{\exp\left(-\frac{\|x-tx_i\|^2}{2(1-t)^2} \right)}{\sum_{j=1}^N\exp\left(-\frac{\|x-tx_j\|^2}{2(1-t)^2} \right)}.
\]
Following notation from \cite{wan2025elucidating}, the empirical optimal denoiser in $t$, is then given by 
\begin{equation}\label{eq:denoiser-t}
m_t\left(x\right):= \mathbb E\left[\bm X \mid \bm X_t = x \right]= \sum_{i=1}^Nw_i(x,t)x_i.
\end{equation}
Then with $\sigma :=(1-t)/t$, by direct calculation, the two empirical optimal denoisers have the following relationship:
for any $x\in\R^d$ and $t\in(0,1)$, one has the following straightforward observation (which was also implicitly shown in the proof of \cite[Proposition 2.2]{wan2025elucidating}):
\begin{equation}\label{eq:deterministic denoiser relation}
     m_t(tx)=m_\sigma(x)
\end{equation}
and consequently, for all $t \in (0,1)$,
\begin{equation}\label{eq:two-denoisers}
    m_t\left(\bm X_t\right) = m_\sigma(\bm X_\sigma).
\end{equation}

The following corollary is designed to help explain the empirical result found in figure 1 of \cite{bertrand2025closed}, where in high dimensions sharp concentration of the cosine similarity between the conditional and empirical optimal vector fields was observed with respect to the flow matching setting with linear scheduling. 
\begin{corollary}\label{cor:cos_sim}
Assume the setting in \Cref{lem:post_weight_epsilon} for $\mathcal D = \{x_1,\dots,x_N\} \subset \mathbb R^d$. Condition on $\bm X= x_1$ and suppose $x_1 \neq 0$. Pick constants $a,c>0$ and fix $\epsilon>0$. Then with probability at least $1-\delta_{\epsilon,a,c}$
\[
\dfrac{\langle u_t(\bm X_t),  u_t(\bm X_t \mid x_1)\rangle}{\|u_t(\bm X_t)\|\|u_t(\bm X_t \mid x_1)\|} \geq 1- \dfrac{2D\epsilon}{(1-t)\sqrt{d-2\sqrt{dc}+\|x_1\|^2-a\|x_1\|}+D\epsilon}
\]
where 
\[
\delta_{\epsilon,a,c}:= e^{-c} + \tilde F\left(\frac{a}{2}\right) + \sum_{j\neq 1}\tilde F \left(\frac{\|\hat d_j\|}{2\sigma}+\frac{\sigma \kappa_\epsilon}{\|\hat d_j\|} \right)
\]
and 
\[
\kappa_\epsilon := \log\left(\frac{\epsilon}{(N-1)(1-\epsilon)}\right)
\]
and $\tilde F := 1-F$ where $F$ is the CDF of the standard normal distribution. 
\end{corollary}
\begin{remark}
The strong dependence on the dimension $d$ should be noted in the bound. For empirical validation we use the CIFAR-10 dataset consisting of 50K images embedded into $[-1,1]^{3072}$ in the standard manner. At $t=0.4$ ( equivalently $\sigma = 1.5$) with $\epsilon=0.01$, $a=8$ , and $c=5$, averaging over a random sample of 500 images, \Cref{cor:cos_sim} yields an average lower bound for the cosine similarity of $0.940$ with an average probability of $0.941$. Though stronger concentration is seen in figure 1 of \cite{bertrand2025closed} for smaller values of $t$, this result conveys the effect of high dimension and dataset diameter on alignment between $u_t(\bm X_t)$ and $u_t(\bm X_t \mid x_1)$. 
\end{remark}
\begin{proof}
Let $\epsilon, a,c >0$. From $\bm X = x_1$ we have $\bm X_t = (1-t)\bm Z + tx_1$. Plugging $\bm X_t$ into \Cref{eq:cond-vec} and \Cref{eq:opt-vect} and then simplifying yields
\[
u_t(\bm X_t \mid x_1) = x_1- \bm Z. 
\]
and
\[
u_t(\bm X_t) = \frac{m_t(\bm X_t)-x_1}{1-t} + x_1 - \bm Z
\]
Hence after expanding
\begin{align}
\dfrac{\langle u_t(\bm X_t),  u_t(\bm X_t \mid x_1)\rangle}{\|u_t(\bm X_t)\|\|u_t(\bm X_t \mid x_1)\|} 
&= \dfrac{\|x_1 - \bm Z\|^2+ \frac{1}{1-t}\left\langle m_t(\bm X_t) - x_1,x_1-\bm Z\right\rangle}{\|\frac{1}{1-t}(m_t(\bm X_t) - x_1) + x_1 - \bm Z\|\|x_1-\bm Z\|}\nonumber\\ 
&\geq \dfrac{\|x_1 - \bm Z\| - \frac{1}{1-t}\|m_t(\bm X_t) - x_1\|}{\|x_1 - \bm Z\| + \frac{1}{1-t}\|m_t(\bm X_t) - x_1\|}\nonumber \\
&= 1- \dfrac{2\|m_t(\bm X_t) - x_1\|}{(1-t)\|x_1 - \bm Z\| + \|m_t(\bm X_t) - x_1\|} \label{eq:cossim_bdd}
\end{align}
Where we used Cauchy--Schwartz and the triangle inequality. 
Recall $x_1 \neq 0$ and define the events
\[
\mathcal F_{a}:= \left \{ \left\langle\bm Z, \frac{x_1}{\|x_1\|}  \right\rangle\leq \frac{a}{2}\right\}
\]
and 
\[
\mathcal G_c := \left\{\|\bm Z\| \geq \sqrt{d-2\sqrt{dc}}\right\}.
\]
So that on $\mathcal F_a \cap \mathcal G_c$, 
\begin{equation}\label{eq:cossim_norm}
    \|x_1 - \bm Z\|^2 = \|\bm Z\|^2 + \|x_1\|^2 - 2\langle x_1, \bm Z\rangle \geq d-2\sqrt{dc} + \|x_1\|^2 - a\|x_1\|.
\end{equation}
Note since $\langle \bm Z, x_1\rangle \sim \mathcal N(0,\|x_1\|^2)$ we have 
\[
\mathbb P(\mathcal F^c_a) = \tilde F \left(\frac{a}{2} \right)
\]
where $\tilde F = 1-F$ and $F$ is the CDF of the standard normal. Furthermore by the Laurent-Massart inequality \cite{laurent_massart_2000}, we have with regrettable notation
\[
\mathbb P(\mathcal G_c^c) \leq e^{-c}.
\]
Let $\sigma := (1-t)/t$ and so by \Cref{eq:two-denoisers} we have
\begin{equation}\label{eq:cossim_equal}
    \|m_t(\bm X_t) - x_1\| = \|m_\sigma(\bm X_\sigma) - x_1\|
\end{equation}
where $\bm X_\sigma = x_1 + \sigma \bm Z$. By the proof of \Cref{lem:post_weight_epsilon} we have 
\begin{equation}\label{eq:cossim_den}
    \|m_\sigma(\bm X_\sigma) - x_1\| \leq D\epsilon
\end{equation}
on an event $\mathcal E$ with 
\[
\mathbb P(\mathcal E^c) \leq \sum_{j\neq 1}\tilde F \left(\frac{\|\hat d_j\|}{2\sigma} + \frac{\sigma \kappa_\epsilon}{\|\hat d_j\|} \right)
\]
where 
\[
\kappa_\epsilon := \log\left(\frac{\epsilon}{(N-1)(1-\epsilon)}\right).
\]
Define the event 
\[
\mathcal H:= \mathcal E \cap \mathcal F_a \cap \mathcal G_c
\]
so that on $\mathcal H$, using the equality in \Cref{eq:cossim_equal} we have by \Cref{eq:cossim_bdd}, \Cref{eq:cossim_norm} and \Cref{eq:cossim_den},
\[
\dfrac{\langle u_t(\bm X_t),  u_t(\bm X_t \mid x_1)\rangle}{\|u_t(\bm X_t)\|\|u_t(\bm X_t \mid x_1)\|} \geq 1- \dfrac{2D\epsilon}{(1-t)\sqrt{d-2\sqrt{dc}+\|x_1\|^2-a\|x_1\|}+D\epsilon}
\]
while using union bound and previous complement bounds yields,
\[
\mathbb P(\mathcal H^c)\leq  e^{-c} + \tilde F\left(\frac{a}{2}\right) + \sum_{j\neq 1}\tilde F \left(\frac{\|\hat d_j\|}{2\sigma}+\frac{\sigma \kappa_\epsilon}{\|\hat d_j\|} \right) := \delta_{\epsilon,a,c}
\]
which completes the proof. 
\end{proof}

\section{Gaussian Shell and Coverage}\label{sec:Gaussian shell theory}
\begin{proof}[Proof of \Cref{lem:one_shell_lemma_main}]
    We first recall a standard concentration inequality for chi-square random variables due to Laurent and Massart \cite{laurent_massart_2000}. 
\begin{lemma}
\label{lem:chi2-LM}
Let $\bm R\sim\chi^2_{d}$.
Then for every $c>0$,
\begin{align}
\mathbb P\!\left(\bm R-d \ge 2\sqrt{dc}+2c\right)
&\le e^{-c}, \label{eq:LM-upper}\\
\mathbb P\!\left(d-\bm R \ge 2\sqrt{dc}\right)
&\le e^{-c}. \label{eq:LM-lower}
\end{align}
\end{lemma}

Now we let $\bm S:=\|\bm Z\|^2\sim\chi^2_d$. By \Cref{lem:chi2-LM} above, for every $c\ge 0$,
\[
\mathbb P\big(\bm S\ge d+2\sqrt{cd}+2c\big)\le e^{-c},
\qquad
\mathbb P\big(\bm S\le d-2\sqrt{cd}\big)\le e^{-c}.
\]
Therefore,
\[
\mathbb P\Big(d-2\sqrt{cd}\le \bm S\le d+2\sqrt{cd}+2c\Big)\ge 1-2e^{-c}.
\]
Since $\big(r^{\mathrm{in}}_{c,d}\big)^2=d-2\sqrt{cd}$ and $\big(r^{\mathrm{out}}_{c,d}\big)^2=d+2\sqrt{cd}+2c$, taking square roots yields
\[
\mathbb P\big(\|\bm Z\|\in[r^{\mathrm{in}}_{c,d},r^{\mathrm{out}}_{c,d}]\big)\ge 1-2e^{-c},
\]
which proves the lemma.
\end{proof}

\subsection{Proof of \Cref{thm:shell-coverage}}
We first introduce some terminology to assist with the proof.
Let $(\Omega,\mathcal F,\mathbb P)$ be a probability space supporting random variables
\[
\bm X,\bm{X}',\bm X_1,\dots,\bm X_N:\Omega\to\R^d,\qquad \bm Z:\Omega\to\R^d,
\]
such that $\bm X,\bm{X}',\bm X_1,\dots,\bm X_N$ are i.i.d.\ with law $p$, and $\bm Z\sim\mathcal N(0,I_d)$ is independent.

We first recall the following result:
\begin{lemma}[Example 4.1.7 in \cite{durrett2019probability}]
\label{lem:durrett}
Let $\bm A$ and $\bm B$ be independent random variables and let
$\varphi$ be an integrable function such that
$\mathbb E[|\varphi(\bm A,\bm B)|]<\infty$.
Define
\[
g(a):=\mathbb E[\varphi(a,\bm B)].
\]
Then
\[
\mathbb E[\varphi(\bm A,\bm B)\mid \bm A]=g(\bm A)
\qquad\text{a.s.}
\]
\end{lemma}

\begin{lemma}
\label{lem:conditional-iid-shell}
For each $i=1,\dots,N$, define
\[
I_i:=\mathbf 1_{\{\bm X_\sigma\in \sh{\sigma}{\bm X_i}\}}.
\]
Then:
\begin{enumerate}
\item[(i)]
\[
\mathbb E[I_i\mid \bm X_\sigma]
=
q_\sigma(\bm X_\sigma)
\qquad\text{a.s.,}
\]
where
\[
q_\sigma(y)
=
\mathbb P\bigl(y\in \sh{\sigma}{\bm X'}\bigr),
\]
and $\bm X'$ is an independent copy of $\bm X_1$.

\item[(ii)]
The family $(I_1,\dots,I_N)$ is conditionally independent given
$\bm X_\sigma$.
\end{enumerate}
Consequently, conditional on $\bm X_\sigma$, the random variables
$I_1,\dots,I_N$ are i.i.d.\ Bernoulli with parameter
$q_\sigma(\bm X_\sigma)$.
\end{lemma}

\begin{proof}
Since $(\bm X_1,\dots,\bm X_N)$ is independent of $(\bm X,\bm Z)$ and
$\bm X_\sigma=\bm X+\sigma\bm Z$ is measurable with respect to
$(\bm X,\bm Z)$, it follows that $(\bm X_1,\dots,\bm X_N)$ is independent
of $\bm X_\sigma$.

\medskip

\noindent\textit{(i)}
Define $\varphi(y,x):=\mathbf 1_{\{y\in \sh{\sigma}{x}\}}$,
$I_i=\varphi(\bm X_\sigma,\bm X_i)$.

Since $\bm X_\sigma$ and $\bm X_i$ are independent and $\varphi$ is bounded,
Lemma~\ref{lem:durrett} yields
\[
\mathbb E[I_i\mid \bm X_\sigma]
=
\mathbb E[\varphi(\bm X_\sigma,\bm X_i)\mid \bm X_\sigma]
=
g(\bm X_\sigma),
\]
where
$g(y)=\mathbb E[\varphi(y,\bm X_i)]
=\mathbb P\bigl(y\in \sh{\sigma}{\bm X_i}\bigr)$.

Since $\bm X_i\stackrel d=\bm X'$, we have $g(y)=q_\sigma(y)$ for all $y$.
Hence
\[
\mathbb E[I_i\mid \bm X_\sigma]=q_\sigma(\bm X_\sigma)
\qquad\text{a.s.}
\]

\noindent\textit{(ii)}
Let $J\subset\{1,\dots,N\}$ be finite and nonempty, and fix
$\varepsilon_j\in\{0,1\}$ for each $j\in J$.
Define the measurable function
\[
\Phi:\ \mathbb R^d\times(\mathbb R^d)^J\to\{0,1\},
\qquad
\Phi\bigl(y,(x_j)_{j\in J}\bigr)
:=
\prod_{j\in J}\mathbf 1_{\{\mathbf 1_{\{y\in \sh{\sigma}{x_j}\}}=\varepsilon_j\}}.
\]
Then
\[
\mathbf 1_{\cap_{j\in J}\{I_j=\varepsilon_j\}}
=
\Phi\bigl(\bm X_\sigma,(\bm X_j)_{j\in J}\bigr).
\]

Since $\bm X_\sigma$ is independent of $(\bm X_j)_{j\in J}$ and $\Phi$ is bounded,
Lemma~\ref{lem:durrett} yields
\[
\mathbb P\!\left(
\bigcap_{j\in J}\{I_j=\varepsilon_j\}
\;\middle|\;
\bm X_\sigma
\right)
=
\mathbb E\!\left[\Phi\bigl(\bm X_\sigma,(\bm X_j)_{j\in J}\bigr)\,\middle|\,\bm X_\sigma\right]
=
G(\bm X_\sigma),
\]
where
\[
G(y)
=
\mathbb E\!\left[
\Phi\bigl(y,(\bm X_j)_{j\in J}\bigr)
\right].
\]
Using independence of $(\bm X_j)_{j\in J}$, we factorize
\begin{align*}
G(y)
&=
\mathbb E\!\left[
\prod_{j\in J}\mathbf 1_{\{\mathbf 1_{\{y\in \sh{\sigma}{\bm X_j}\}}=\varepsilon_j\}}
\right]
=
\prod_{j\in J}
\mathbb E\!\left[
\mathbf 1_{\{\mathbf 1_{\{y\in \sh{\sigma}{\bm X_j}\}}=\varepsilon_j\}}
\right] \\
&=
\prod_{j\in J}
\mathbb P\!\left(\mathbf 1_{\{y\in \sh{\sigma}{\bm X_j}\}}=\varepsilon_j\right)
=
\prod_{j\in J}
\mathbb P\!\left(I_j=\varepsilon_j \,\middle|\, \bm X_\sigma=y\right).
\end{align*}
Hence
\[
\mathbb P\!\left(
\bigcap_{j\in J}\{I_j=\varepsilon_j\}
\;\middle|\;
\bm X_\sigma
\right)
=
\prod_{j\in J}\mathbb P\!\left(I_j=\varepsilon_j \,\middle|\, \bm X_\sigma\right),
\]
which proves that $(I_1,\dots,I_N)$ is conditionally independent given $\bm X_\sigma$.

Finally, by part~(i),
\[
\mathbb P(I_i=1\mid \bm X_\sigma)
=
\mathbb E[I_i\mid \bm X_\sigma]
=
q_\sigma(\bm X_\sigma),
\]
so $(I_1,\dots,I_N)$ are conditionally i.i.d.\ Bernoulli with parameter
$q_\sigma(\bm X_\sigma)$.
\end{proof}

\begin{lemma}
\label{lem:two-shell-distance}
Fix $\sigma>0$. For any deterministic $x,x'\in\R^d$, define $\bm Y_x=x+\sigma \bm Z$ with $\bm Z\sim\mathcal N(0,I_d)$.
Then
\[
\mathbb P\big(\bm Y_x\in \sh{\sigma}{x}\cap \sh{\sigma}{x'}\big)
=
\Phi_{d,c}\!\left(\frac{\|x-x'\|}{\sigma}\right).
\]
\end{lemma}

\begin{proof}
Let $\Delta:=x-x'$ and $t:=\|\Delta\|/\sigma$. Note that
\[
\bm Y_x\in \sh{\sigma}{x}\iff \|\bm Y_x-x\|\in[\sigma r^{\mathrm{in}}_{c,d},\sigma r^{\mathrm{out}}_{c,d}]\iff \|\bm Z\|\in[r^{\mathrm{in}}_{c,d},r^{\mathrm{out}}_{c,d}],
\]
and
\[
\bm Y_x\in \sh{\sigma}{x'}\iff \|\bm Y_x-x'\|\in[\sigma r^{\mathrm{in}}_{c,d},\sigma r^{\mathrm{out}}_{c,d}]\iff \|\bm Z+\Delta/\sigma\|\in[r^{\mathrm{in}}_{c,d},r^{\mathrm{out}}_{c,d}].
\]
Choose an orthogonal matrix $Q\in O(d)$ such that $Q(\Delta/\sigma)=t e_1$.
Since $\bm Z\sim\mathcal N(0,I_d)$ is rotation invariant, $Q\bm Z\stackrel{d}{=}\bm Z$.
Therefore,
\begin{align*}
&\mathbb P\Big(\|\bm Z\|\in[r^{\mathrm{in}}_{c,d},r^{\mathrm{out}}_{c,d}],\ \|\bm Z+\Delta/\sigma\|\in[r^{\mathrm{in}}_{c,d},r^{\mathrm{out}}_{c,d}]\Big)
\\
&=
\mathbb P\Big(\|Q\bm Z\|\in[r^{\mathrm{in}}_{c,d},r^{\mathrm{out}}_{c,d}],\ \|Q\bm Z+Q(\Delta/\sigma)\|\in[r^{\mathrm{in}}_{c,d},r^{\mathrm{out}}_{c,d}]\Big)\\
&=
\mathbb P\Big(\|\bm Z\|\in[r^{\mathrm{in}}_{c,d},r^{\mathrm{out}}_{c,d}],\ \|\bm Z+t e_1\|\in[r^{\mathrm{in}}_{c,d},r^{\mathrm{out}}_{c,d}]\Big)
=
\Phi_{d,c}(t).
\end{align*}
\end{proof}

\begin{lemma}
\label{lem:random-centers-to-Phi}
Let $(\Omega,\mathcal F,\mathbb P)$ be a probability space supporting random vectors
$\bm X,\bm X':\Omega\to\R^d$ and $\bm Z:\Omega\to\R^d$.
Assume $\bm Z\sim\mathcal N(0,I_d)$ and that $\bm Z$ is independent of $(\bm X,\bm X')$.
Define the noisy point $\bm X_\sigma:=\bm X+\sigma \bm Z$ and the indicator
\[
U:=\mathbf 1_{\{\bm X_\sigma\in \sh{\sigma}{\bm X}\cap \sh{\sigma}{\bm X'}\}}
=\mathbf 1_{\{\bm X+\sigma \bm Z\in \sh{\sigma}{\bm X}\cap \sh{\sigma}{\bm X'}\}}.
\]

Then
\[
\mathbb E\big[U\mid \bm X,\bm X'\big]
=
\Phi_{d,c}\!\left(\frac{\|\bm X-\bm X'\|}{\sigma}\right)
\quad\text{a.s.}
\]
\end{lemma}

\begin{proof}
Define
\[
\Psi(x,x',z)
:=
\mathbf 1_{\{x+\sigma z\in \sh{\sigma}{x}\cap \sh{\sigma}{x'}\}},
\qquad (x,x',z)\in\mathbb R^d\times\mathbb R^d\times\mathbb R^d.
\]
Then $U=\Psi( \bm X,\bm X',\bm Z)$.

For fixed $(x,x')$, by independence and the definition of $\sh{\sigma}{x}$,
\[
\mathbb E_Z[\Psi(x,x',\bm Z)]
=
\mathbb P\!\left(
\|\bm Z\|\in[r^{\mathrm{in}}_{c,d},r^{\mathrm{out}}_{c,d}],\,
\left\|\bm Z+\tfrac{x-x'}{\sigma}\right\|\in[r^{\mathrm{in}}_{c,d},r^{\mathrm{out}}_{c,d}]
\right).
\]

By \Cref{lem:two-shell-distance}, this probability depends
only on $\|x-x'\|/\sigma$, and equals
\[
\Phi_{d,c}\!\left(\frac{\|x-x'\|}{\sigma}\right).
\]

Since $\bm Z$ is independent of $(\bm X,\bm X')$, we conclude by \Cref{lem:durrett} again that
\[
\mathbb E[U\mid \bm X,\bm X']
=
\Phi_{d,c}\!\left(\frac{\|\bm X-\bm X'\|}{\sigma}\right)
\quad\text{a.s.}
\]
\end{proof}

Now we are ready to prove our theorem.

\begin{proof}[Proof of {\Cref{thm:shell-coverage}}]
By definition,
\[
\mathbf 1_{\{\bm X_\sigma\notin\mathcal U_\sigma\}}
=
\prod_{i=1}^N \mathbf 1_{\{\bm X_\sigma\notin \sh{\sigma}{\bm X_i}\}}
=
\prod_{i=1}^N (1-I_i),
\qquad
I_i:=\mathbf 1_{\{\bm X_\sigma\in \sh{\sigma}{\bm X_i}\}}.
\]
Using
Lemma~\ref{lem:conditional-iid-shell}, we obtain almost surely
\[
\mathbb E\!\left[\mathbf 1_{\{\bm X_\sigma\notin\mathcal U_\sigma\}}\mid \bm X_\sigma\right]
=
\prod_{i=1}^N
\mathbb E\!\left[1-I_i\mid \bm X_\sigma\right]
=
\prod_{i=1}^N (1-q_\sigma(\bm X_\sigma))
=
(1-q_\sigma(\bm X_\sigma))^N.
\]
Therefore,
\[
\mathbb E\!\left[\mathbf 1_{\{\bm X_\sigma\in\mathcal U_\sigma\}}\mid \bm X_\sigma\right]
=
1-(1-q_\sigma(\bm X_\sigma))^N.
\]
Applying the tower property of conditional expectation yields
\[
\mathbb P(\bm X_\sigma\in\mathcal U_\sigma)
=
\mathbb E\!\left[\mathbf 1_{\{\bm X_\sigma\in\mathcal U_\sigma\}}\right]
=
\mathbb E\!\left[
\mathbb E\!\left[\mathbf 1_{\{\bm X_\sigma\in\mathcal U_\sigma\}}\mid \bm X_\sigma\right]
\right]
=
\mathbb E\!\left[1-(1-q_\sigma(\bm X_\sigma))^N\right].
\]

Now, we define the nearest-neighbor index
\[
i^\star(\omega)
:=
\min\Big\{i:\ \|\bm X(\omega)-\bm X_i(\omega)\|=d_{1\mathrm{NN}}(\bm X(\omega))\Big\},
\]
which is measurable by construction, and note that
$\sh{\sigma}{\bm X_{i^\star}}\subseteq\mathcal U_\sigma$ pointwise.
Hence,
\[
\mathbf 1_{\{\bm X_\sigma\in\mathcal U_\sigma\}}
\ge
\mathbf 1_{\{\bm X_\sigma\in \sh{\sigma}{\bm X}\cap \sh{\sigma}{\bm X_{i^\star}}}\}.
\]
Taking expectations and using the definition of probability as expectation, we obtain
\[
\mathbb P(\bm X_\sigma\in\mathcal U_\sigma)
\ge
\mathbb E\!\left[
\mathbf 1_{\{\bm X_\sigma\in \sh{\sigma}{\bm X}\cap \sh{\sigma}{\bm X_{i^\star}}}\}
\right].
\]

Now consider the conditional expectation with respect to
$\bm X,\bm X_{i^\star}$.
By \Cref{lem:random-centers-to-Phi},
\[
\mathbb E\!\left[
\mathbf 1_{\{\bm X_\sigma\in \sh{\sigma}{\bm X}\cap \sh{\sigma}{\bm X_{i^\star}}}\}
\;\middle|\;
\bm X,\bm X_{i^\star}
\right]
=
\Phi_{d,c}\!\left(\frac{\|\bm X-\bm X_{i^\star}\|}{\sigma}\right)
=
\Phi_{d,c}\!\left(\frac{d_{1\mathrm{NN}}(\bm X)}{\sigma}\right).
\]
Applying the tower property once more gives
\[
\mathbb P(\bm X_\sigma\in\mathcal U_\sigma)
\ge
\mathbb E\!\left[
\Phi_{d,c}\!\left(\frac{d_{1\mathrm{NN}}(\bm X)}{\sigma}\right)
\right],
\]
which establishes the lower bound.

Write
\[
\mathbf 1_{\{\bm X_\sigma\in\mathcal U_\sigma\}}
\le
\mathbf 1_{E^c}
+
\mathbf 1_{E}\mathbf 1_{\{\bm X_\sigma\in\mathcal U_\sigma\}},
\qquad
E:=\{\bm X_\sigma\in \sh{\sigma}{\bm X}\}.
\]
Taking expectations yields
\[
\mathbb P(\bm X_\sigma\in\mathcal U_\sigma)
\le
\mathbb P(E^c)+\mathbb P(E\cap\{\bm X_\sigma\in\mathcal U_\sigma\}).
\]
By Lemma~\ref{lem:one_shell_lemma_main},
\(
\mathbb P(E^c)=\mathbb P(\|\bm Z\|\notin[r^{\mathrm{in}}_{c,d},r^{\mathrm{out}}_{c,d}])\le 2e^{-c}.
\)

On the event $E$, if $\bm X_\sigma\in\mathcal U_\sigma$ then $\bm X_\sigma\in \sh{\sigma}{\bm X_i}$ for at least one $i$.
Thus,
\[
\mathbf 1_E\mathbf 1_{\{\bm X_\sigma\in\mathcal U_\sigma\}}
\le
\sum_{i=1}^N \mathbf 1_{E}\mathbf 1_{\{\bm X_\sigma\in \sh{\sigma}{\bm X_i}}\}.
\]
Taking expectations and using linearity, we have that
\[
\mathbb P(E\cap\{\bm X_\sigma\in\mathcal U_\sigma\})
\le
\sum_{i=1}^N \mathbb P(E\cap\{\bm X_\sigma\in \sh{\sigma}{\bm X_i}\}).
\]

Fix $i$.
By \Cref{lem:random-centers-to-Phi} and the tower property,
\[
\mathbb P(E\cap\{\bm{X}_\sigma\in \sh{\sigma}{\bm X_i}\})
=
\mathbb E\!\left[
\mathbb E\!\left[
\mathbf 1_{\{\bm{X}_\sigma\in \sh{\sigma}{\bm X}\cap \sh{\sigma}{\bm X_i}}\}
\;\middle|\;
\bm{X},\bm{X}_i
\right]
\right]
=
\mathbb E\!\left[
\Phi_{d,c}\!\left(\frac{\|\bm{X}-\bm{X}_i\|}{\sigma}\right)
\right].
\]
Since $\bm{X}_i$ is independent of $\bm{X}$ and $\bm{X}_i\stackrel{d}{=}\bm{X}'$, this equals
\(\mathbb E[\Phi_{d,c}(\|\bm{X}-\bm{X}'\|/\sigma)]\).
Summing over $i$ and adding the bound on $\mathbb P(E^c)$ yields
\[
\mathbb P(\bm{X}_\sigma\in\mathcal U_\sigma)
\le
2e^{-c}
+
N\,\mathbb E\!\left[
\Phi_{d,c}\!\left(\frac{\|\bm{X}-\bm{X}'\|}{\sigma}\right)
\right],
\]
which completes the proof.
\end{proof}

\subsection{Proof of \Cref{thm:inf-many-minimizers-disjoint-shell-informal}}

We formalize that, when training is effectively restricted to disjoint Gaussian shells, the denoising objective is highly non-identifiable and admits infinitely many global optima.
To isolate what the objective does (and does not) constrain on each shell, we idealize the Gaussian corruption by replacing the radially concentrated law of $Z$ with a \emph{uniform distribution} on the corresponding unit shell $\mathcal S:=\{z\in\mathbb R^d:\ r^{\mathrm{in}}_{c,d}\le \|z\|\le r^{\mathrm{out}}_{c,d}\}$.

\begin{theorem}
\label{thm:inf-many-minimizers-disjoint-shell}
Fix $\sigma>0$ and a finite dataset $\mathcal D=\{x_1,\dots,x_N\}\subset\mathbb R^d$.
Let $\sh{\sigma}{x_i}$ be the Gaussian shells from Section~\ref{Gaussianshells}, and
assume they are pairwise disjoint:
\[
\sh{\sigma}{x_i}\cap \sh{\sigma}{x_j}=\emptyset
\qquad \text{for all  } i\neq j.
\]
Let $\bm Z\sim\mathrm{Unif}(\mathcal S)$, where $\mathcal S:=\{z\in\mathbb R^d:\ r^{\mathrm{in}}_{c,d}\le \|z\|\le r^{\mathrm{out}}_{c,d}\}$,
and consider the shell-only objective
\[
\mathcal L_\sigma(m)
:= \mathbb E_{\substack{i\sim \mathrm{Unif}([N])\\ \bm Z\sim \mathrm{Unif}(\mathcal S)}}
\Bigl[\bigl\|m(x_i+\sigma \bm Z)-x_i\bigr\|_2^2\Bigr]
\]
over measurable $m:\mathbb R^d\to\mathbb R^d$. Then, there are infinitely many global minimizers: any $m$ such that
\[
m(y)=x_i \quad \text{for all } y\in \sh{\sigma}{x_i},\ \ i=1,\dots,N
\]
is a minimizer, and $m$ can be defined arbitrarily on
$\mathbb R^d\setminus\bigcup_{i=1}^N \sh{\sigma}{x_i}$.
\end{theorem}

\begin{proof}
We first note that $\mathcal L_\sigma(m)\ge 0$ for every measurable $m$.

Let $m$ be any measurable map satisfying
\[
m(y)=x_i\quad\text{for all }y\in \sh{\sigma}{x_i},\ \ i=1,\dots,N.
\]
Then for every $i\in[N]$ and every $\bm Z$ with $\|\bm Z\|\in[r^{\mathrm{in}}_{c,d},r^{\mathrm{out}}_{c,d}]$, we again have
$x_i+\sigma \bm Z\in \sh{\sigma}{x_i}$ and therefore
\[
m(x_i+\sigma \bm Z)=x_i,
\qquad
\|m(x_i+\sigma \bm Z)-x_i\|_2^2=0.
\]
Taking expectation shows $\mathcal L_\sigma(m)=0$, hence every such $m$ is a global minimizer.
\end{proof}

\section{Denoiser Behavior at Large Noise Levels}\label{sec:large noise taylor}
In this section, we prove Theorem~\ref{thm:gauss-plus-excess}. Instead of working with the $\sigma$ parameter directly, we first work with a parameter $t\in[0,1]$ from flow matching as already mentioned in \Cref{sec:cosine-similarity}.

We use $\bm X\sim p$ to denote a random variable. Since we have assumed $p$ to have bounded support, we further assume that there is $R>0$ such that
$\mathrm{supp}(p)\subset B_R(0)$. We next rewrite the denoiser $m_t$ (cf. \Cref{eq:denoiser-t}) as follows.

For $x\in\R^d$ and $t\in(-1,1)$, define
\begin{equation}\label{eq:def-wDNm}
W_t(y;x)\;:=\;\exp\!\Big(-\frac{\|x-ty\|^2}{2(1-t)^2}\Big),\qquad
D_t(x)\;:=\;\E[W_t(\bm X;x)],\qquad
N_t(x)\;:=\;\E[\bm X\,W_t(\bm X;x)].
\end{equation}
Then, we have that for any $t\in[0,1)$
\begin{equation}\label{eq:def-mt}
m_t(x)\;=\mathbb{E}[\bm X\mid\bm X_t=x]=\;\frac{N_t(x)}{D_t(x)}.
\end{equation}
Note that although $\mathbb{E}[\bm X\mid\bm X_t=x]$ is only defined for $t\in[0,1)$, the formula $N_t(x)/D_t(x)$ is well-defined for all $t\in(-1,1)$, and we will work with this extension of $m_t$ in the proof.
Note that $0<W_t(y;x)\le 1$ for all $t\in(-1,1)$, so $D_t(x)\in(0,1]$.

We begin by establishing a collection of lemmas that will be used in the proof of \Cref{thm:gauss-plus-excess}.

\begin{lemma}
\label{lem:DN-C2}
Fix $\delta\in(0,1)$.
Then for every $x\in\R^d$ the functions
\[
t\longmapsto D_t(x),\qquad t\longmapsto N_t(x)
\]
are $C^2$ on $[-\delta,\delta]$.
Moreover, for $k=0,1,2$,
\[
(x,t)\longmapsto \partial_t^k D_t(x),
\qquad
(x,t)\longmapsto \partial_t^k N_t(x)
\]
are continuous on $\R^d\times[-\delta,\delta]$.
Moreover,
\[
(x,t)\longmapsto \partial_t^2 m_t(x)
\]
is continuous on $\R^d\times[-\delta,\delta]$.
\end{lemma}

\begin{proof}
Write
\[
W_t(y;x)=e^{F(t,x,y)},
\qquad
F(t,x,y)=-\frac{\|x-ty\|^2}{2(1-t)^2}.
\]
The function $F$ is smooth on $(-1,1)\times\R^d\times\R^d$.
For $k=0,1,2$,
\[
\partial_t^k W_t(y;x)=P_k(x,y,t)\,W_t(y;x),
\]
where $P_k$ is rational in $t$ with the denominators of the form $(1-t)^{-m}$, and polynomial in $(x,y)$.

Fix $a>0$.
Since $|t|\le\delta<1$, the factors $(1-t)^{-m}$ are uniformly bounded.
Hence there exists $C>0$ such that
\[
\sup_{\substack{\|x\|\le a\\ |t|\le\delta\\ \|y\|\le R}}
|\partial_t^k W_t(y;x)|\le C.
\]
Because $\|\bm X\|\le R$ a.s., we obtain
\[
|\partial_t^k W_t(\bm X;x)|\le C,
\qquad
\|\bm X\,\partial_t^k W_t(\bm X;x)\|\le RC
\quad\text{a.s.}
\]
By dominated convergence,
\[
\partial_t^k D_t(x)=\E[\partial_t^k W_t(\bm X;x)],
\qquad
\partial_t^k N_t(x)=\E[\bm X\,\partial_t^k W_t(\bm X;x)],
\]
and these depend continuously on $(x,t)$.
We have $m_t(x)=N_t(x)/D_t(x)$ and $D_t(x)>0$.
Since $N_t$ and $D_t$ are $C^2$ in $t$ with continuous derivatives.
The quotient rule gives the claim that $\partial_t^2 m_t(x)$ is continuous on $\R^d\times[-\delta,\delta]$.
\end{proof}

Now let $\mu:=\E[\bm X]$ and $\Sigma:=\mathrm{Cov}(\bm X)$. Let $G=\mathcal N(\mu,\Sigma)$ and define
$m_t^G(x)$ by the same formula \Cref{eq:def-wDNm}--\Cref{eq:def-mt} but with expectation taken under $G$.

Recall from \Cref{eq:Gaussian closed form} that in the Gaussian case
\(Y\sim\mathcal N(\mu,\Sigma)\) with \(\Sigma\succeq 0\), we have a closed-form
expression for \(m_\sigma^G(x)\).
Using \Cref{eq:deterministic denoiser relation}, this yields the following
closed form for \(m_t^G\) for every \(t\in(-1,1)\):
\begin{equation}\label{eq:mtG-closed-form}
m_t^G(x)
=
\mu
+
t\,\Sigma\bigl(t^2\Sigma+(1-t)^2 I\bigr)^{-1}(x-t\,\mu).
\end{equation}
This formula follows from the conditional mean formula for multivariate
Gaussian distributions.

Then it is not hard to prove  the following lemma.
\begin{lemma}[Gaussian smoothness of $m_t^G$]\label{lem:gaussian-mt-closed-form}
 For each fixed $x\in\mathbb R^d$, the map
$t\mapsto m_t^G(x)$ is $C^\infty$ on $(-1,1)$. Moreover,
for any $a>0$ and any $\delta\in(0,1)$, the map $(x,t)\mapsto m_t^G(x)$ is $C^\infty$
on $B_a(0)\times[-\delta,\delta]$.
\end{lemma}

\begin{lemma}\label{lem:first-two-coeff}
For every fixed $x\in\R^d$,
\[
m_0(x)=\mu,\qquad \left.\frac{d}{dt}m_t(x)\right|_{t=0}=\Sigma x.
\]
The same identities hold for $m_t^G(x)$ (with the same $\mu,\Sigma$).
\end{lemma}

\begin{proof}
At $t=0$, the weight is constant in $y$:
\[
W_0(y;x)=\exp(-\|x\|^2/2)=:W_0(x).
\]
Therefore $N_0(x)=W_0(x)\E[\bm X]=W_0(x)\mu$ and $D_0(x)=W_0(x)$, giving $m_0(x)=\mu$.

For the first derivative, differentiate $m_t=N_t/D_t$ 
\begin{equation}\label{eq:added}
m_t'=\frac{N_t'}{D_t}-\frac{N_t}{D_t}\frac{D_t'}{D_t}.   
\end{equation}
It suffices to compute $N_0',D_0'$. A direct calculation yields
\[
\left.\partial_t \log W_t(y;x)\right|_{t=0}=-\|x\|^2+\langle x,y\rangle,
\qquad
\left.\partial_t W_t(y;x)\right|_{t=0}=W_0(x)\,(-\|x\|^2+\langle x,y\rangle).
\]
Hence
\[
D_0'(x)=W_0(x)\E[-\|x\|^2+\langle x,\bm X\rangle],\qquad
N_0'(x)=W_0(x)\E\big[\bm X(-\|x\|^2+\langle x,\bm X\rangle)\big].
\]
Substituting into the \cref{eq:added} at $t=0$ and using $m_0(x)=\mu$ gives
\[
m_0'(x)
=\E\big[\bm X\langle x,\bm X\rangle\big]-\mu\,\E[\langle x,\bm X\rangle]
=\E\big[(\bm X-\mu)\langle x,(\bm X-\mu)\rangle\big]
=\Sigma x.
\]
The same computation applies to $G=\mathcal N(\mu,\Sigma)$, so the identities hold for $m_t^G$ as well.
\end{proof}

\begin{theorem}
\label{thm:gauss-plus-excess-assist}
Assume that $\supp(p)\subset B_R(0)$.  
Then for all $\delta\in(0,1)$, there exists a bounded function
\[
\tilde H:\overline{B_R(0)}\times[-\delta,\delta]\to\mathbb{R}^d
\]
such that for all $x\in\overline{B_R(0)}$ and all $|t|\le\delta$,
\[
m_t(x)=m_t^G(x)+t^2 \tilde H(x,t).
\]
\end{theorem}

\begin{proof}
For $x\in\overline{B_R(0)}$, define
\[
f_x(t):=m_t(x)-m_t^G(x).
\]
By Lemma~\ref{lem:DN-C2}, the function $f_x$ is of class $C^2$ in $t$.  
Moreover, Lemma~\ref{lem:first-two-coeff} implies that
\[
f_x(0)=f_x'(0)=0.
\]
Again by Lemma~\ref{lem:DN-C2}, the map $(x,t)\mapsto f_x''(t)$ is continuous on
$\overline{B_R(0)}\times[-\delta,\delta]$.

Define
\[
\tilde H(x,t):=\int_0^1 (1-s)\,f_x''(st)\,ds.
\]
Then the map $(x,t,s)\mapsto (1-s)f_x''(st)$ is continuous on the compact set
\[
\overline{B_R(0)}\times[-\delta,\delta]\times[0,1],
\]
and therefore $\tilde H$ is continuous on
$\overline{B_R(0)}\times[-\delta,\delta]$.
By compactness, $\tilde H$ is uniformly bounded on this set.

Applying Taylor's theorem with integral remainder, we obtain for all
$|t|\le\delta$,
\[
f_x(t)=t^2\int_0^1 (1-s)f_x''(st)\,ds
= t^2 \tilde H(x,t).
\]
Consequently,
\[
m_t(x)=m_t^G(x)+t^2 \tilde H(x,t),
\]
which completes the proof.
\end{proof}

Now, we are going to translate the second-order expansion in the small parameter $t$ into an asymptotic expansion in terms of the scale parameter $\sigma$.

\begin{proof}[Proof of Theorem~\ref{thm:gauss-plus-excess}]
    Let $\sigma=\frac{1-t}{t}$, then, by \Cref{eq:deterministic denoiser relation} we have that
\[
m_\sigma(x)=m_t(t x),
\qquad
m_\sigma^G(x)=m_t^G(t x).
\] 
By \Cref{thm:gauss-plus-excess-assist}, for any $|t|\le\delta$ and $x\in\overline{B_R(0)}$,
\[
m_t(x)=m_t^G(x)+t^2 \tilde H(x,t).
\]
Replacing $x$ by $t x$, we obtain
\[
m_t(t x)=m_t^G(t x)+t^2 \tilde H(t x,t).
\]
By the relationship between $m_t$ and $m_\sigma$, we have
\[
m_\sigma(x)=m_\sigma^G(x)+t(\sigma)^2 \tilde H(t(\sigma)x,t(\sigma)).
\]

Define 
\[
H(x,\sigma):=\tilde H(t(\sigma)x,t(\sigma)).
\]
Since $t=\frac{1}{1+\sigma}\to 0$ as $\sigma\to\infty$, there exists $\sigma_0>0$ such that
$|t(\sigma)|\le 1$  for all $\sigma\ge\sigma_0$.

\[
\|tx\|\le \|x\|\le R,
\]
so $tx\in\overline{B_R(0)}$.

Therefore, for $\sigma\ge\sigma_0$,
$(tx,t)$ lies in
$\overline{B_R(0)}\times[-\delta,\delta]$, where $\tilde H$ is bounded.
Hence $H$ is bounded on
$\overline{B_R(0)}\times[\sigma_0,\infty)$.

Finally, since $t=(1+\sigma)^{-1}$, then $t^2=(1+\sigma)^{-2}$,
which yields the stated form.
\end{proof}

\section{Sensitivity Localization}
\label{app:sensitivity-localization}

We analyze the case when the data distribution $p$ is a Gaussian distribution $p=\mathcal{N}(\mu,\Sigma)$ on $\R^d$ to  establish a spatial sensitivity decay control. Recall from \Cref{eq:Gaussian closed form} that the optimal denoiser for $p$ is given by the closed-form
\[m_\sigma(x)=\mu+\Sigma(\Sigma+\sigma^2 I)^{-1}(x-\mu).\]

In the following we only consider the case of a single channel.
The multivariate case can be handled by the same arguments by treating each channel separately.
We think of $d=r\times r$ where $r$ is the resolution and then each dimension represents a single-valued pixel.
To quantify how much the denoiser at pixel $i$ depends on observation $x_j$, we
study the Jacobian $\nabla_x m_\sigma(x)$. 
By Tweedie's formula~\citep{efron2011tweedie}, we have that
\[\nabla_x m_\sigma(x)=\frac{1}{\sigma^2}\,\mathrm{Cov}(\bm X \mid \bm{X}_\sigma=x)= \Sigma(\Sigma+\sigma^2 I)^{-1}.\]
The expression for $\nabla_x m_\sigma(y)$ was derived in~\cite{lukoianov2025locality}, who also observed that $\nabla_x m_\sigma(y)\to I$ as $\sigma\to 0$.
However, convergence to the identity does not by itself characterize how the off-diagonal entries decay with spatial distance.

We now make the spatial decay claim precise under cyclic stationary.

\begin{lemma}\label{lem:ij formula}
    For any $i\neq j$, we have that
    \begin{equation}
    \label{eq:offdiag-id-simple}
    \left|\frac{\partial [m_\sigma]_i}{\partial x_j}\right|
    =\sigma^2\left|[(\Sigma+\sigma^2 I)^{-1}]_{ij}\right|.
    \end{equation}
\end{lemma}
\begin{proof}
     The identity \eqref{eq:offdiag-id-simple} follows from the simple identity
    $\Sigma(\Sigma+\sigma^2 I)^{-1}=I-\sigma^2(\Sigma+\sigma^2 I)^{-1}$ and the fact that $I_{ij}=0$ for $i\neq j$.
\end{proof}

\begin{definition}[Circulant covariance]
A matrix $\Sigma\in\R^{d\times d}$ is called \emph{circulant} if there exists a vector
$c=[c_0,\dots,c_{d-1}]^T$ such that
\[
\Sigma_{ij}=c_{(i-j)\bmod d}, \qquad i,j\in\{0,\dots,d-1\}.
\]
Equivalently, each row of $\Sigma$ is a cyclic shift of its first row. 
\end{definition}

\begin{lemma}
\label{lem:circulant-spectrum}
Let $\Sigma\in\mathbb R^{d\times d}$ be circulant and let
$\omega=e^{2\pi i/d}$. Then $\Sigma$ is diagonalized by the discrete Fourier
basis $(\omega^{jk})_{j,k=0}^{d-1}$, with eigenvalues
$(\lambda_k)_{k=0}^{d-1}$. In particular, for any function $f$ applied
spectrally to $\Sigma$, the matrix $f(\Sigma)$ is circulant and its entries depend only on the wrap-around distance
$\mathrm{dist}(i,j):=\min\{|i-j|,\,d-|i-j|\}$.
\end{lemma}

\begin{proof}[Proof sketch]
This is a classical property of circulant matrices: they are simultaneously
diagonalized by the discrete Fourier transform, and functional calculus
preserves circulant structure. See \cite{Gray2005Toeplitz} for a detailed
treatment of the spectral theory of circulant and Toeplitz matrices.
\end{proof}

\begin{theorem}
\label{thm:Gaussian-sensitivity-decay}
Assume $\Sigma$ is circulant. Let $\omega=e^{2\pi i/d}$ and let $(\lambda_k)_{k=0}^{d-1}$ be the
eigenvalues of $\Sigma$ in the discrete Fourier basis (Lemma~\ref{lem:circulant-spectrum}). Define
\[
h_k:=\frac{\sigma^2}{\lambda_k+\sigma^2},\qquad
q_r:=\frac{1}{d}\sum_{k=0}^{d-1} h_k\,\omega^{rk},
\qquad
r\in\{0,\dots,d-1\}.
\]
For all $i\neq j$, let $n:=\min\{|i-j|,\,d-|i-j|\}$ be the wrap-around distance. Then,
\[
\left|\frac{\partial [m_\sigma]_i}{\partial x_j}\right|=|q_n|.
\]
Moreover, for all $n\in\{1,\dots,\lfloor d/2\rfloor\}$,
\begin{equation}
\label{eq:1overm-TV-simple}
\left|\frac{\partial [m_\sigma]_i}{\partial x_j}\right|
\le \frac{\mathrm{TV}(h)}{4n},
\qquad
\mathrm{TV}(h):=\sum_{k=0}^{d-1}|h_{k+1}-h_k|
\ \ (\text{indices mod } d).
\end{equation}
\end{theorem}

\begin{proof}
Define the matrix
\[
Q_\sigma:=\sigma^2(\Sigma+\sigma^2 I)^{-1}.
\]
Since $\Sigma$ is circulant, so is $\Sigma+\sigma^2 I$, and hence $(\Sigma+\sigma^2 I)^{-1}$
and $Q_\sigma$ are circulant as well (Lemma~\ref{lem:circulant-spectrum} applied to the spectral
function $f(t)=1/(t+\sigma^2)$). In the DFT basis, the eigenvalues of $Q_\sigma$ are
\[
h_k=\frac{\sigma^2}{\lambda_k+\sigma^2},
\qquad k=0,\dots,d-1.
\]
Therefore $Q_\sigma$ has circulant kernel $(q_r)_{r=0}^{d-1}$ given by the inverse DFT:
\[
[Q_\sigma]_{ij}=q_{(i-j)\bmod d},\qquad
q_r=\frac{1}{d}\sum_{k=0}^{d-1} h_k\,\omega^{rk}.
\]

Now fix $i\neq j$ and let $n:=\min\{|i-j|,\,d-|i-j|\}$. By Lemma~\ref{lem:ij formula},
\[
\left|\frac{\partial [m_\sigma]_i}{\partial x_j}\right|
=\sigma^2\left|[(\Sigma+\sigma^2 I)^{-1}]_{ij}\right|
=\left|[Q_\sigma]_{ij}\right|
=|q_n|,
\]
which proves the first claim.

For the decay bound, let $(\Delta h)_k:=h_{k+1}-h_k$ (indices modulo $d$). For $n\not\equiv 0\pmod d$,
a cyclic summation-by-parts identity (discrete integration by parts) yields
\[
q_n=\frac{\omega^n}{d(1-\omega^n)}\sum_{k=0}^{d-1}(\Delta h)_k\,\omega^{nk}.
\]
Taking absolute values and using the triangle inequality gives
\[
|q_n|
\le \frac{\sum_{k=0}^{d-1}|\Delta h_k|}{d\,|1-\omega^n|}
=\frac{\mathrm{TV}(h)}{d\,|1-\omega^n|}.
\]
Finally, for $1\le n\le d/2$ we have
\[
|1-\omega^n|=2|\sin(\pi n/d)|\ge \frac{4n}{d},
\]
and substituting into the previous display yields \eqref{eq:1overm-TV-simple}.
\end{proof}

\section{Experiment Details}

\subsection{Sampling Schedule}
\label{app:sampling_schedule}
Unless otherwise specified, the sampling schedule used in our experiments is the polynomial sampling schedule from~\cite{karras2022elucidating}:
\[
\sigma_i = \Big(\sigma_{\max}^{1/\rho} + \frac{i}{N-1}\big(\sigma_{\min}^{1/\rho}-\sigma_{\max}^{1/\rho}\big)\Big)^{\rho},
\]
for $i=0,\ldots,N-1$, where $\sigma_{\max} =80$, $\sigma_{\min}=0.002$ and $N=18$ is a popular choice for sampling on CIFAR-10.

\subsection{Empirical Shell Coverage on CIFAR-10}\label{exp:coverage_details}

\begin{figure}[t]
    \centering
    \includegraphics[width=0.7\linewidth]{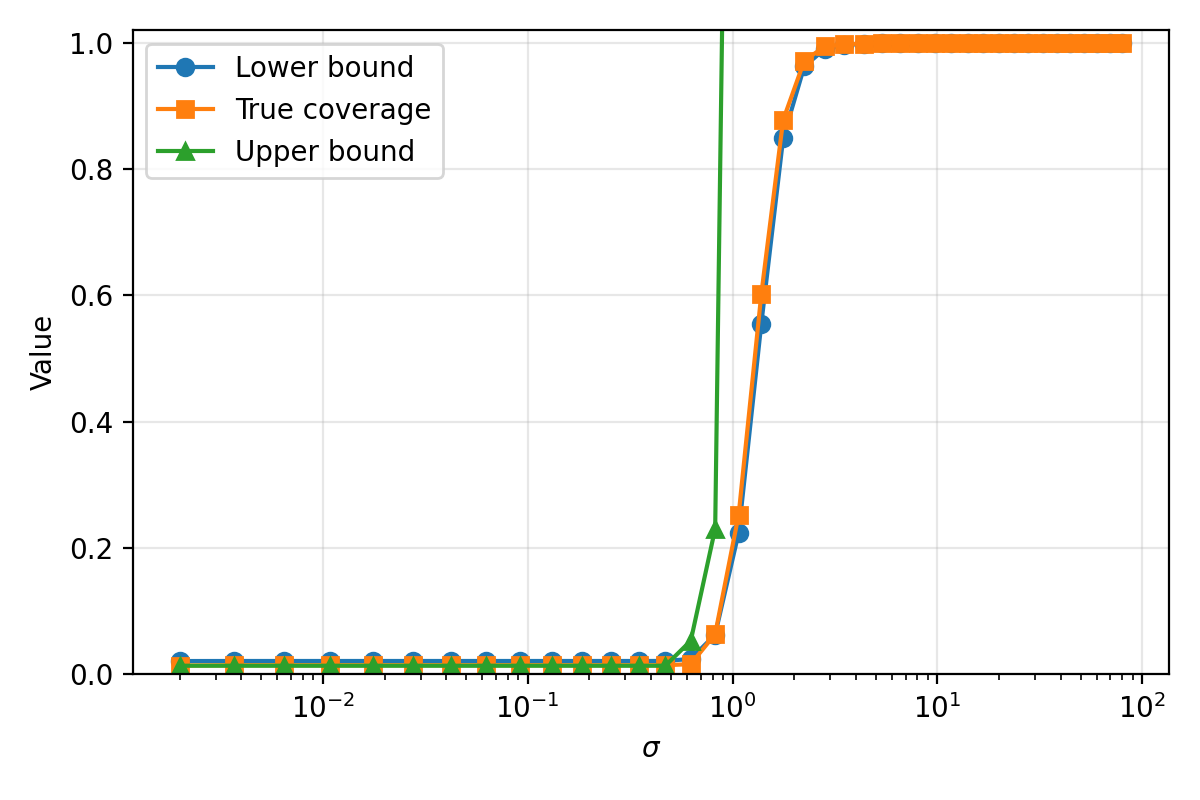}
   \caption{
\textbf{Empirical shell coverage on CIFAR-10.} We plot the lower bound, coverage, and upper bound from Theorem~\ref{thm:shell-coverage}. 
The lower bound curve (blue) is nearly indistinguishable from the coverage curve across the entire range of $\sigma$, as it almost perfectly overlaps with it. 
}
    \label{fig:shell-coverage-cifar10}
\end{figure}

We empirically evaluate the lower and upper bounds in \Cref{thm:shell-coverage} and compare them with the Gaussian shell coverage probability on the CIFAR-10 dataset; see \Cref{fig:shell-coverage-cifar10}.
All images are flattened to vectors in $\mathbb{R}^d$ with $d=3072$ and linearly rescaled to lie in $[-1,1]^d$.
For a fixed shell width parameter $c=5$, we estimate the function
\[
\Phi_{d,c}(t)
=
\mathbb{P}\!\left(
\|\bm Z\|\in[a_d,b_d],\;
\|\bm Z+t e_1\|\in[a_d,b_d]
\right),
\qquad \bm Z\sim\mathcal{N}(0,I_d),
\]
using Monte Carlo with $5000$ i.i.d.\ Gaussian samples.
We reuse the same Gaussian samples to evaluate $\Phi_{d,c}(t)$ for all values of $t$ encountered in the experiment.

\paragraph{Noise schedule.} We use the EDM polynomial noise schedule stated in \Cref{app:sampling_schedule} with $\sigma_{\max}=80$, $\sigma_{\min}=0.002$, and a 40-step discretization ($N=40$).

\paragraph{Upper bound.}
To estimate the upper bound term
$
N\,\mathbb{E}[\Phi_{d,c}(\|\bm X-\bm X'\|/\sigma)]
$
with $N=1000$, we independently sample $12{,}000$ pairs $(\bm X,\bm X')$ from the CIFAR-10 training set and compute the empirical average of $\Phi_{d,c}(\|\bm X-\bm X'\|/\sigma)$ for each noise level $\sigma$.

\paragraph{Lower bound.}
To estimate the lower bound term
$
\mathbb{E}[\Phi_{d,c}(d_{1\mathrm{NN}}(\bm X)/\sigma)]
$,
we perform $250$ independent trials.
In each trial, we sample a fresh subset $\{\bm X_1,\dots,\bm X_N\}$ of $N=1000$ training images and independently sample $8$ test images $\bm X$.
For each test image, we compute the nearest-neighbor distance
$
d_{1\mathrm{NN}}(\bm X)=\min_{1\le i\le N}\|\bm X-\bm X_i\|
$
and evaluate $\Phi_{d,c}(d_{1\mathrm{NN}}(\bm X)/\sigma)$.
The results are averaged over test images and trials.

\paragraph{Gaussian shell coverage.}
We estimate the true coverage probability
$
\mathbb{P}(\bm X_\sigma\in\mathcal{U}_\sigma)
$
using $350$ independent trials.
In each trial, we sample a fresh subset $\{\bm X_1,\dots,\bm X_N\}$ of size $N=1000$ and an independent test image $\bm X$.
For each noise level $\sigma$, we draw $10$ independent Gaussian noises $\bm Z$ and check whether
$
\bm X+\sigma \bm Z
$
lies in at least one shell centered at the subset points.
The coverage probability is estimated by averaging over noises and trials.

\subsection{Max vs Self Posterior Weight}\label{sec:max-vs-w1-weights}
\begin{figure}[H]
    \centering
    \includegraphics[width=0.7\linewidth]{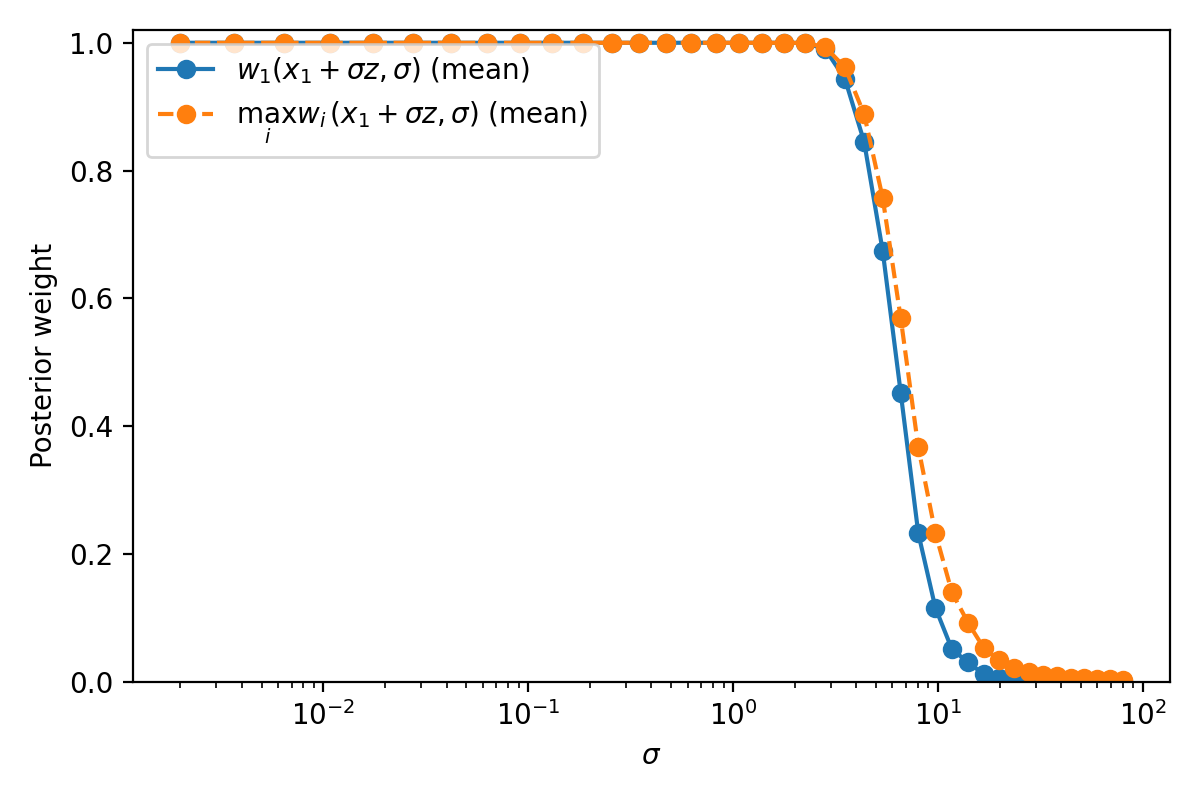}
   \caption{
\textbf{Max vs. Sample Weight on 1K CIFAR-10.}
We plot average values of $\max_{1\leq i\leq N}w_i(x_1 + \sigma \bm Z, \sigma)$ and $w_1(x_1+\sigma \bm Z)$ across a range of $\sigma$. The two curves show that the max posterior weight can usually be well estimated by the sample weight, particularly at smaller noise levels where they coincide.}
\label{fig:max-vs-w1-weights}
\end{figure}
\paragraph{Max vs Self weight estimation.}
For each noise level $\sigma$,
we fix a subset of $N=1000$ images sampled uniformly without replacement from the CIFAR-10 dataset. We then draw 100 base points without replacement from the subset and for each base point $x_1$ we compute for each $i\in \{1,\dots,N\}$, $w_i(x_1+\sigma z,\sigma)$ over 400 independent Gaussian noises $z$. We then compute the averages for both  $\max_iw_i(x_1+\sigma z, \sigma)$ and $w_1(x_1+\sigma z, \sigma)$ over all base points $x_1$ and subsequent samples $z$. The resulting curves show the averages for each value of $\sigma$.

\subsection{Danger Zone Identification: 2k Car Subset}
\label{app:danger_zone}
In Figure~\ref{fig:danger-zone}, we overlay the coverage and max posterior weight curves from Figure~\ref{fig:true-vs-maxweight} for the 2k car training set and 1k car test set. The danger zone is identified as the region where both metrics are high, which in the case of $2k$-car is a narrow interval containing the intersection point $\sigma \approx 1.9$. 
\begin{figure}[htbp!]
\centering
\includegraphics[width=0.7\textwidth]{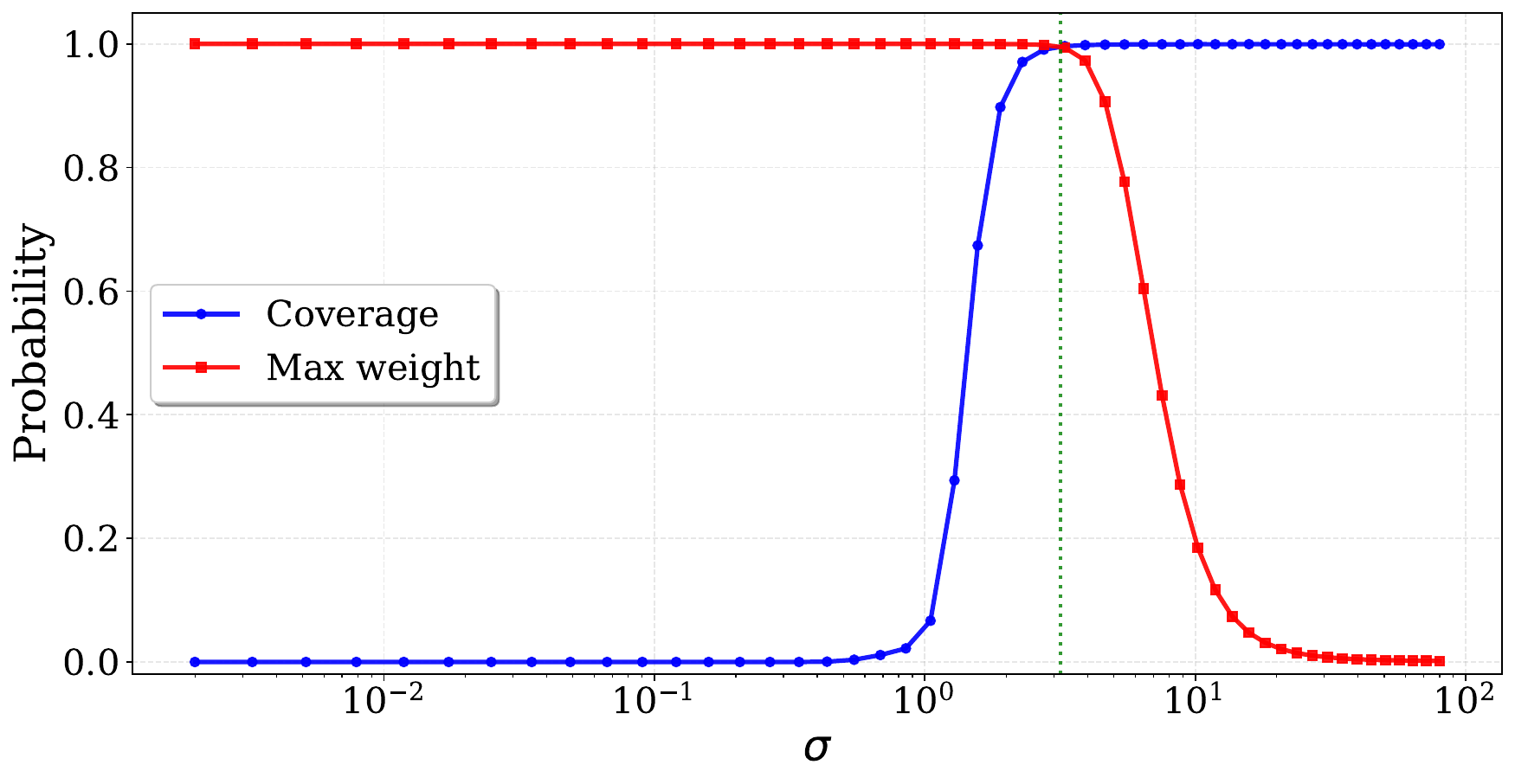}
\caption{Coverage and max posterior weight vs noise level for 2k car training set and 1k car test set. The danger zone is where both metrics are high, spanning $\sigma \in [1, 5]$ with intersection near $\sigma \approx 1.9$.}
\label{fig:danger-zone}
\end{figure}

\subsection{Pareto Frontier of Training Configurations on 1024 Grayscale CelebA Images}
\label{app:pareto}

We compute the FID score of 10k samples against 10k held out test images. We visualize the Pareto frontier of different timestep gap configurations on the 1024 grayscale CelebA training set in Figure~\ref{fig:pareto}. Each point represents a different gap configuration, plotting FID score against memorization rate (FMEM\%). The dashed line connects non-dominated solutions forming the Pareto frontier.
\begin{figure}[h]
\centering
\includegraphics[width=0.85\textwidth]{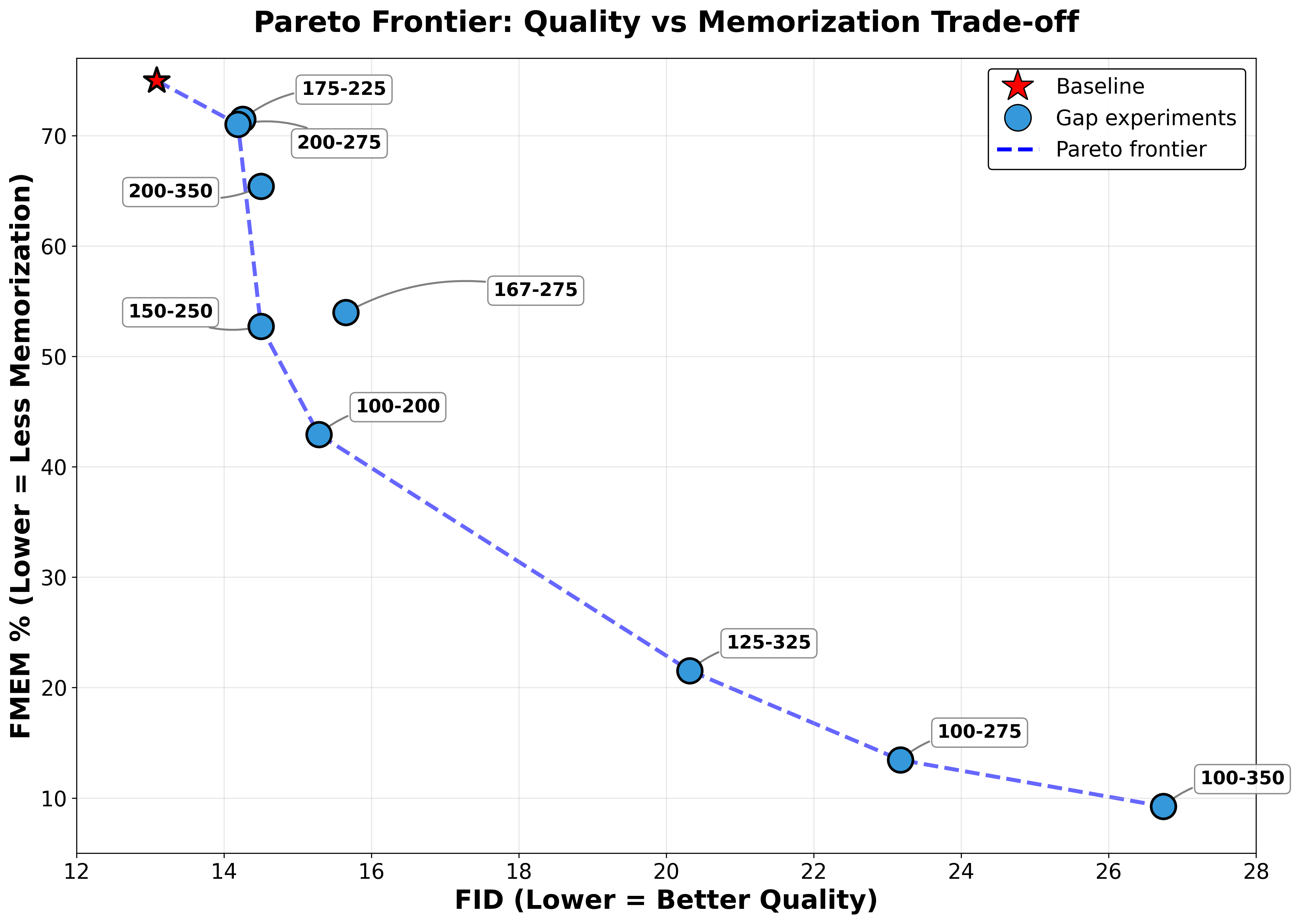}
\caption{Pareto frontier showing the trade-off between image quality (FID) and memorization (FMEM\%). Each point represents a different timestep gap configuration. The dashed line connects non-dominated solutions forming the Pareto frontier.}
\label{fig:pareto}
\end{figure}

\subsection{Denoiser Swapping Experiments}
\label{app:denoiser_swap}

To understand which noise levels are critical for sample quality, we perform denoiser swapping experiments. We generate samples using hybrid schedules that switch between \textbf{EDM-1K} and \textbf{EDM-50K} at different $\sigma$ regions (large: $\sigma > 5.3$, medium: $\sigma \in [0.06, 5.3]$, small: $\sigma < 0.06$).

Figure~\ref{fig:mem_sample_grid} shows representative samples from each condition. Key findings:
\begin{itemize}
    \item Swapping to \textbf{EDM-50K} in the \textbf{medium region} (steps 7-13) produces the largest quality improvement when starting from \textbf{EDM-1K}.
    \item Swapping to \textbf{EDM-1K} in the \textbf{medium region} causes the largest quality degradation when starting from \textbf{EDM-50K}.
    \item large and small region swaps have minimal impact on final sample quality.
\end{itemize}

This confirms that the mid-$\sigma$ region is where memorization most impacts generation quality, aligning with our one-step memorization analysis. Table~\ref{tab:swap_memorization} in the main text quantifies the memorization rates across all conditions.

\begin{figure}[h]
    \centering
    \includegraphics[width=\linewidth]{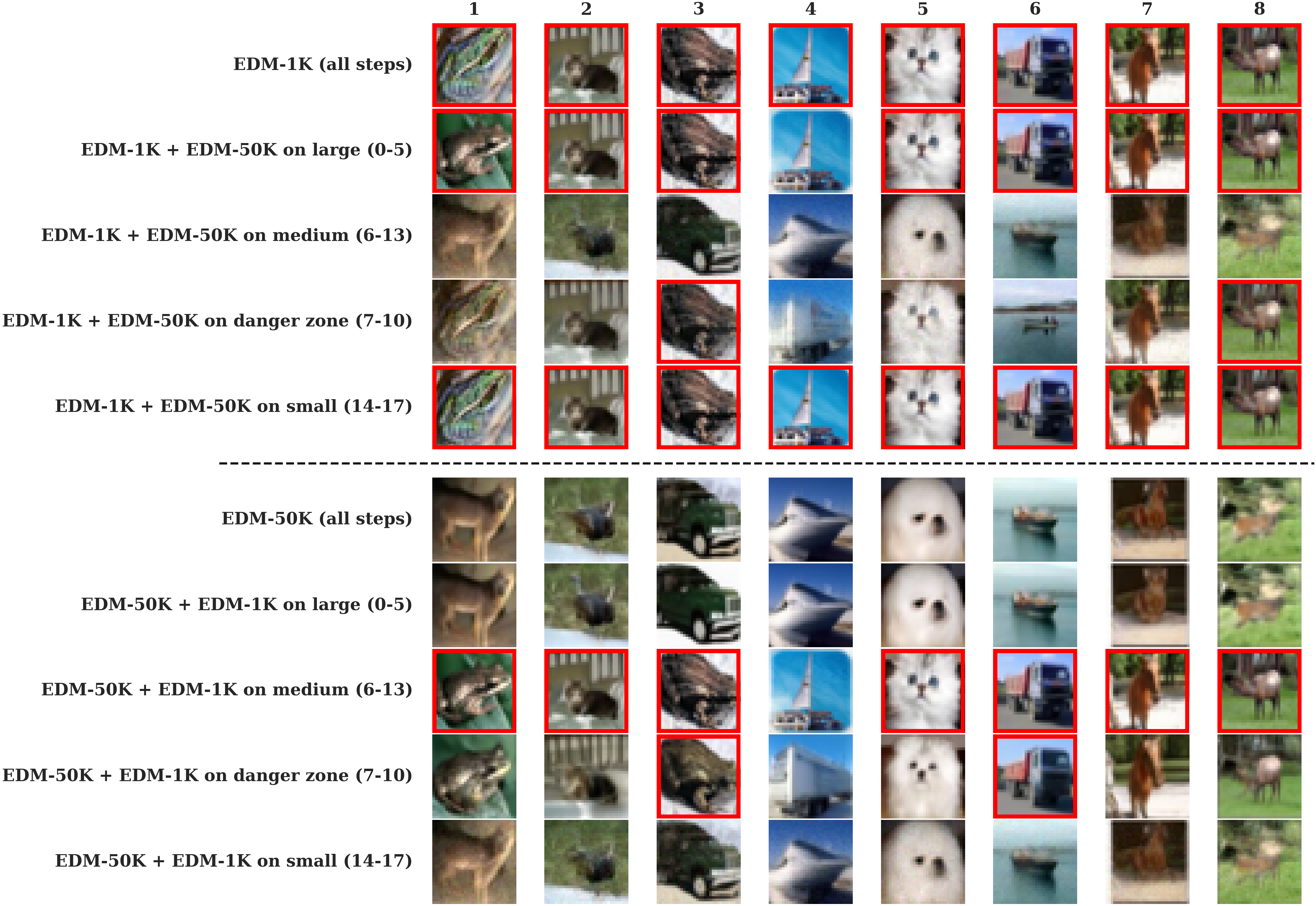}
    \caption{\textbf{Denoiser Swapping Samples.} Top half: \textbf{EDM-1K} base with \textbf{EDM-50K} swapped in different regions. Bottom half: \textbf{EDM-50K} base with \textbf{EDM-1K} swapped in. Red boxes indicate memorized samples (where $d_{1\mathrm{NN}} < d_{2\mathrm{NN}}/3$). Swapping in the medium region (steps 6-13) completely flips memorization behavior, while swapping only in the danger zone (steps 7-10) shows partial effects.}
    \label{fig:mem_sample_grid}
\end{figure}

\subsection{One-Step Memorization Showcase}
\label{app:onestep_showcase}

\begin{table*}[htbp!]
    \centering
    \caption{\textbf{Time Parameterization Conventions for EDM Sampling.} EDM uses 18 sampling steps for CIFAR-10 with noise scales $\sigma$ (top row) and EDM time coordinate  Red denotes initial (high noise) steps, blue denotes middle steps, green denotes end (low noise) steps.}
    \small
    \setlength{\tabcolsep}{4pt}
    \begin{tabular}{l|*{6}{>{\columncolor{red!20}}r}*{7}{>{\columncolor{blue!20}}r}*{5}{>{\columncolor{green!20}}r}}
    \toprule
    $\sigma$ & 80.00 & 57.59 & 40.79 & 28.37 & 19.35 & 12.91 & 8.40 & 5.32 & 3.26 & 1.92 & 1.09 & 0.59 & 0.30 & 0.14 & 0.06 & 0.02 & 0.01 & 0.00 \\
    \bottomrule
    \end{tabular}
    \label{tab:time_conventions}
    \end{table*}
We use the standard 18 steps of EDM for sampling with Euler method, the noise level $\sigma$ and step translation can be seen in~\Cref{app:onestep_showcase}. 
Figures~\ref{fig:mem_onestep_showcase_8.4}--\ref{fig:mem_onestep_showcase_0.14} visualize one-step denoising output at six different noise levels spanning the full diffusion schedule. Each figure shows, from top to bottom: the initial clean test image, noisy input, \textbf{EDM-1K} output followed by its 1-NN in the training set, \textbf{EDM-50K} output followed by its 1-NN, and \textbf{EMP-1K} output followed by its 1-NN.

\begin{figure}[htbp]
    \centering
    \includegraphics[width=\linewidth]{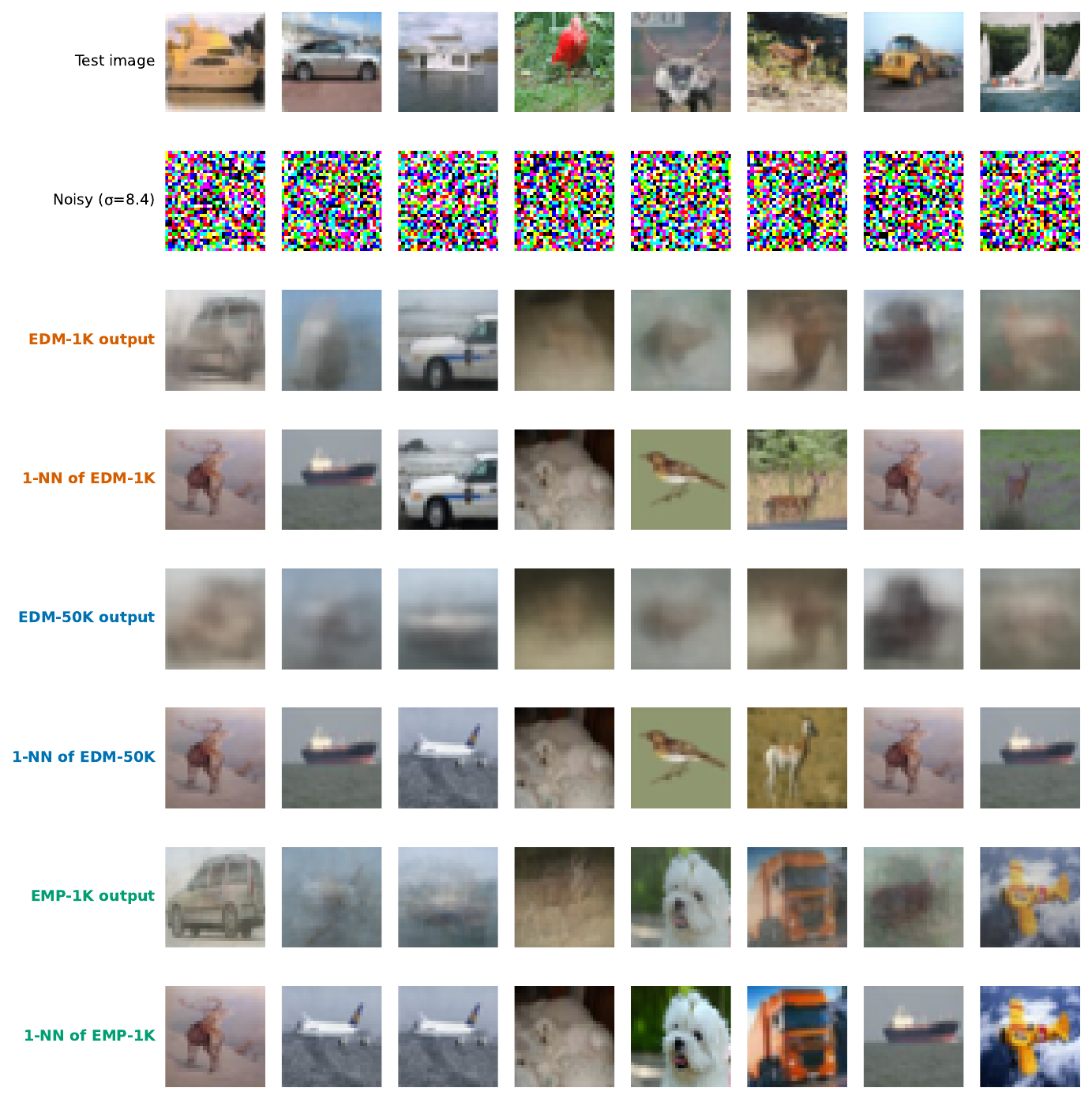}
    \caption{\textbf{One-Step Denoising at $\sigma = 8.4$.} From top: initial test image, noisy input, \textbf{EDM-1K} output, 1-NN of \textbf{EDM-1K}, \textbf{EDM-50K} output, 1-NN of \textbf{EDM-50K}, \textbf{EMP-1K} output, 1-NN of \textbf{EMP-1K}. At this high noise level, all denoisers perform similarly with no memorization.}
    \label{fig:mem_onestep_showcase_8.4}
\end{figure}

\begin{figure}[htbp]
    \centering
    \includegraphics[width=\linewidth]{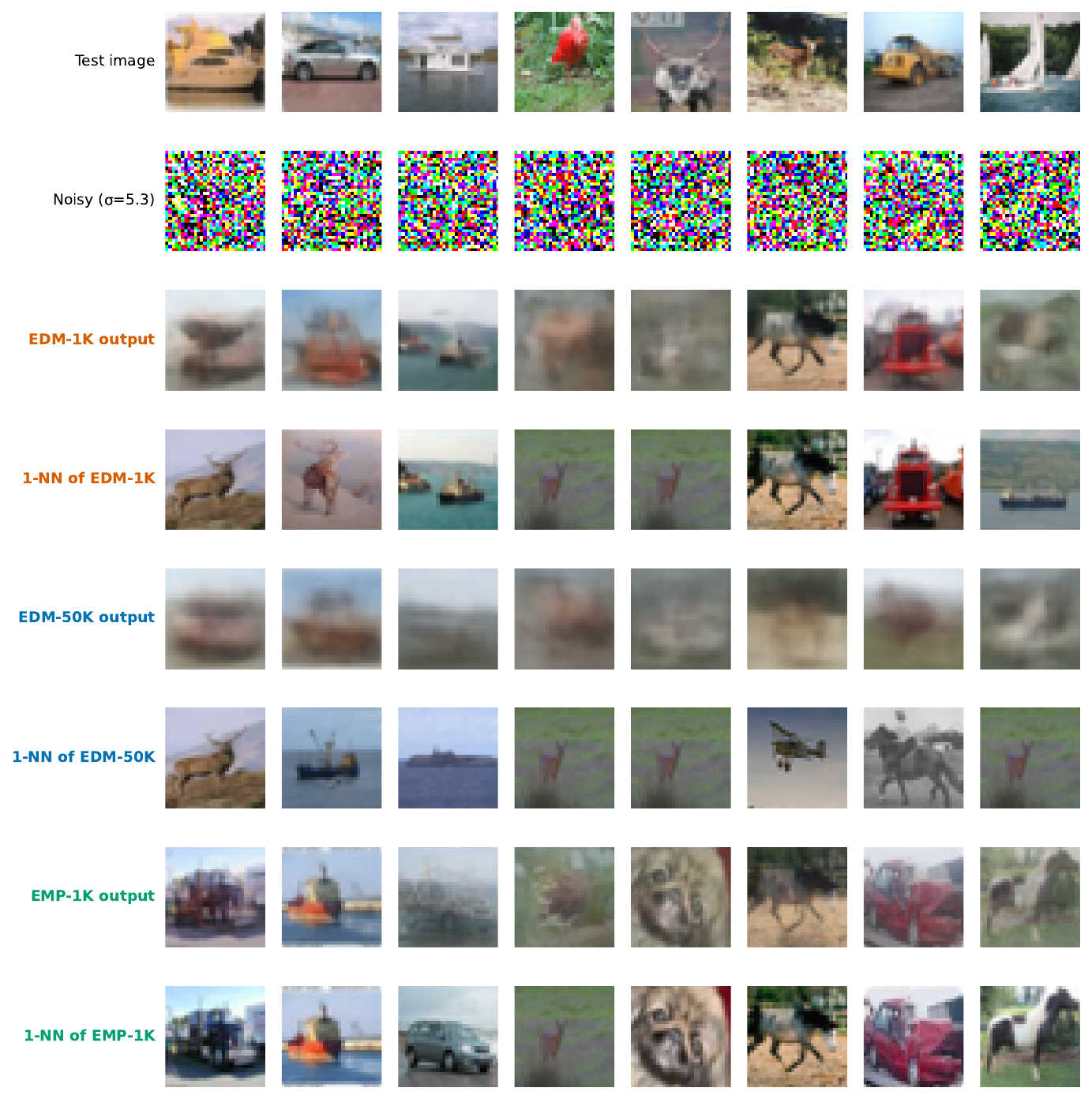}
    \caption{\textbf{One-Step Denoising at $\sigma = 5.3$.} From top: initial test image, noisy input, \textbf{EDM-1K} output, 1-NN of \textbf{EDM-1K}, \textbf{EDM-50K} output, 1-NN of \textbf{EDM-50K}, \textbf{EMP-1K} output, 1-NN of \textbf{EMP-1K}. Denoising begins to show model differences as \textbf{EDM-1K} starts exhibiting some preference for training images.}
    \label{fig:mem_onestep_showcase_5.3}
\end{figure}

\begin{figure}[htbp]
    \centering
    \includegraphics[width=\linewidth]{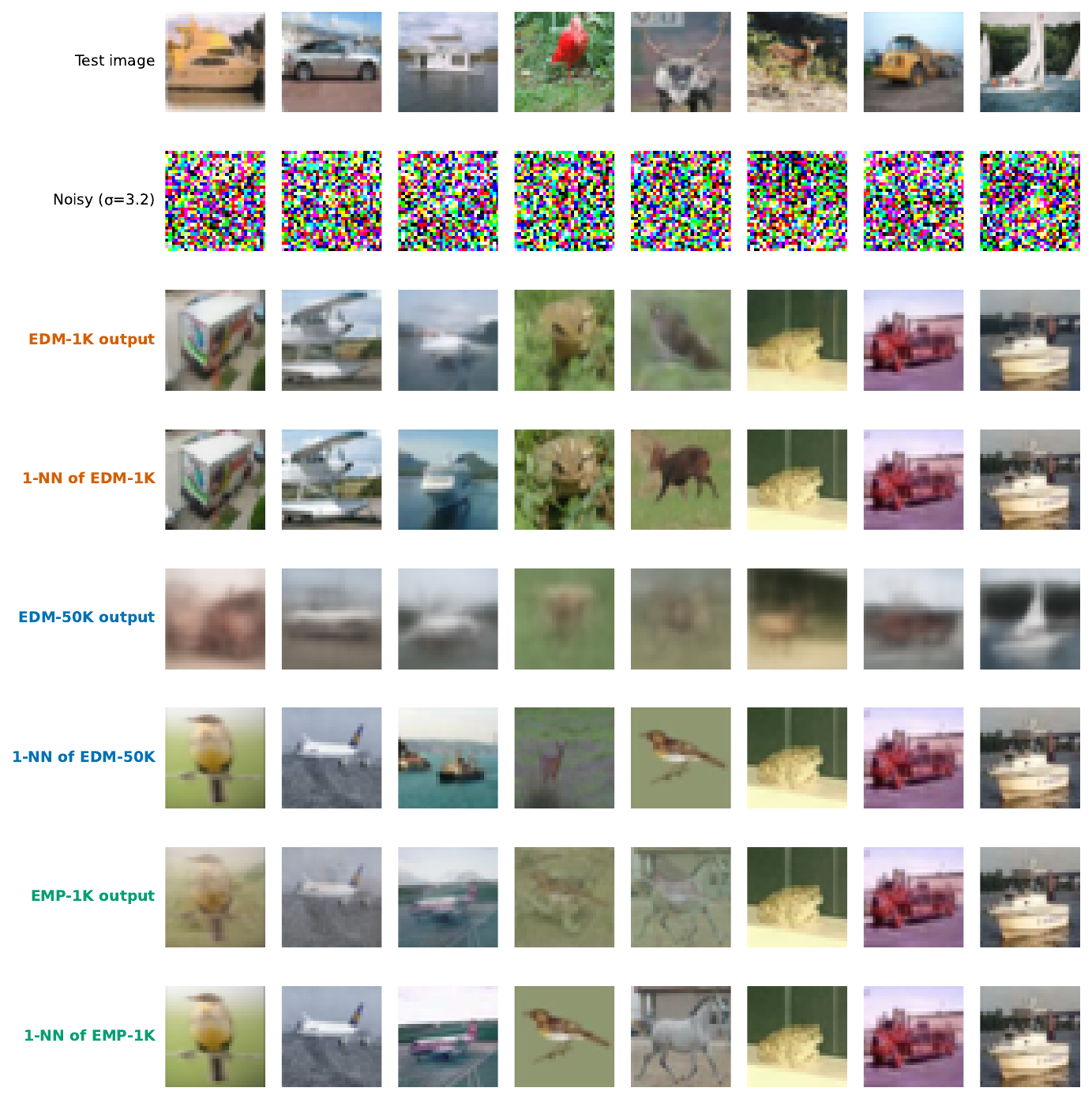}
    \caption{\textbf{One-Step Denoising at $\sigma = 3.2$ .} From top: initial test image, noisy input, \textbf{EDM-1K} output, 1-NN of \textbf{EDM-1K}, \textbf{EDM-50K} output, 1-NN of \textbf{EDM-50K}, \textbf{EMP-1K} output, 1-NN of \textbf{EMP-1K}. \textbf{EDM-1K} memorization becomes pronounced, with outputs matching training nearest neighbors. This is near the peak memorization region.}
    \label{fig:mem_onestep_showcase_3.2}
\end{figure}

\begin{figure}[htbp]
    \centering
    \includegraphics[width=\linewidth]{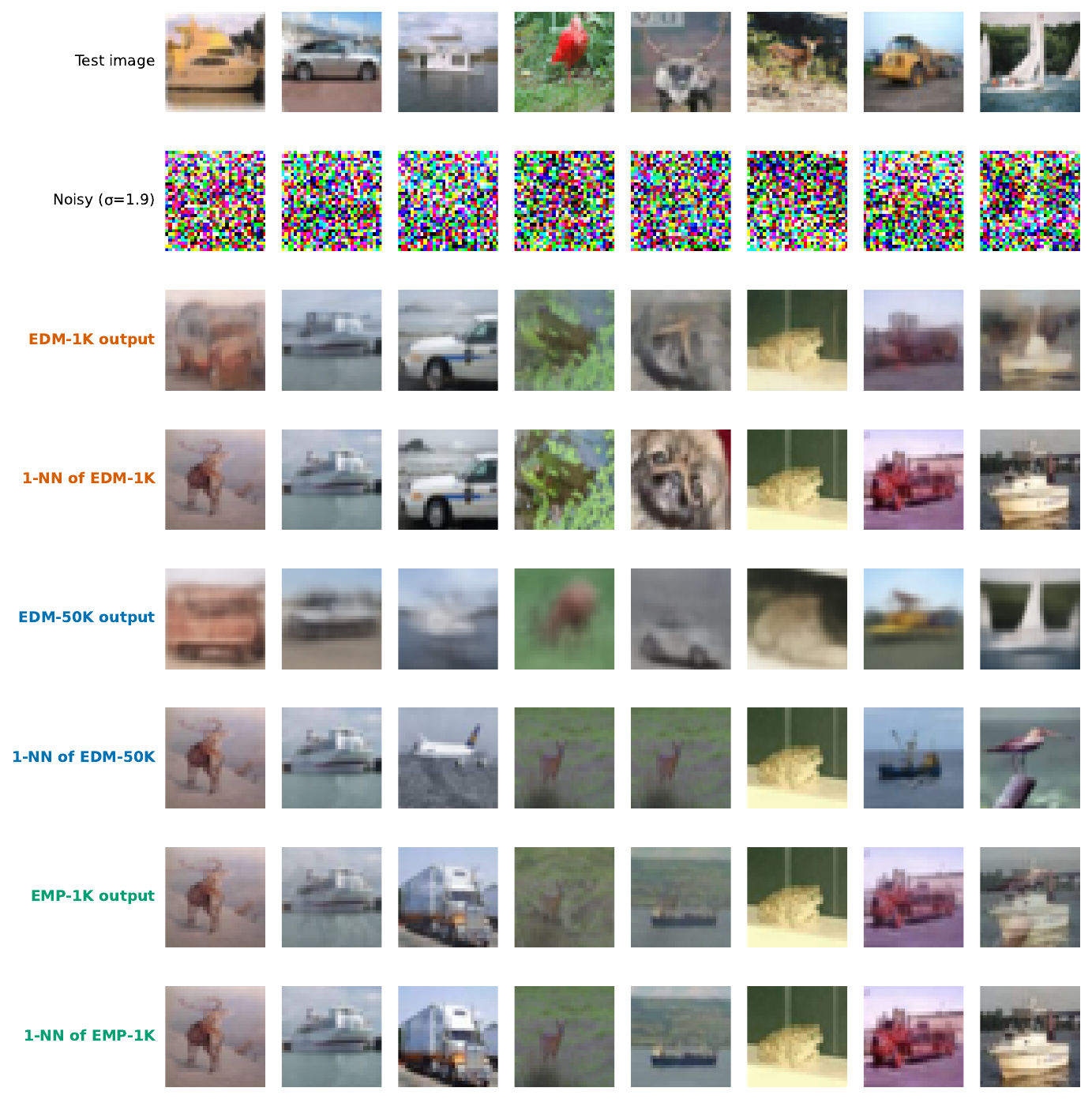}
    \caption{\textbf{One-Step Denoising at $\sigma = 1.9$.} From top: initial test image, noisy input, \textbf{EDM-1K} output, 1-NN of \textbf{EDM-1K}, \textbf{EDM-50K} output, 1-NN of \textbf{EDM-50K}, \textbf{EMP-1K} output, 1-NN of \textbf{EMP-1K}. \textbf{EDM-1K} outputs are much closer to training images than \textbf{EDM-50K} at this critical mid-$\sigma$ level.}
    \label{fig:mem_onestep_showcase_1.9}
\end{figure}

\begin{figure}[htbp]
    \centering
    \includegraphics[width=\linewidth]{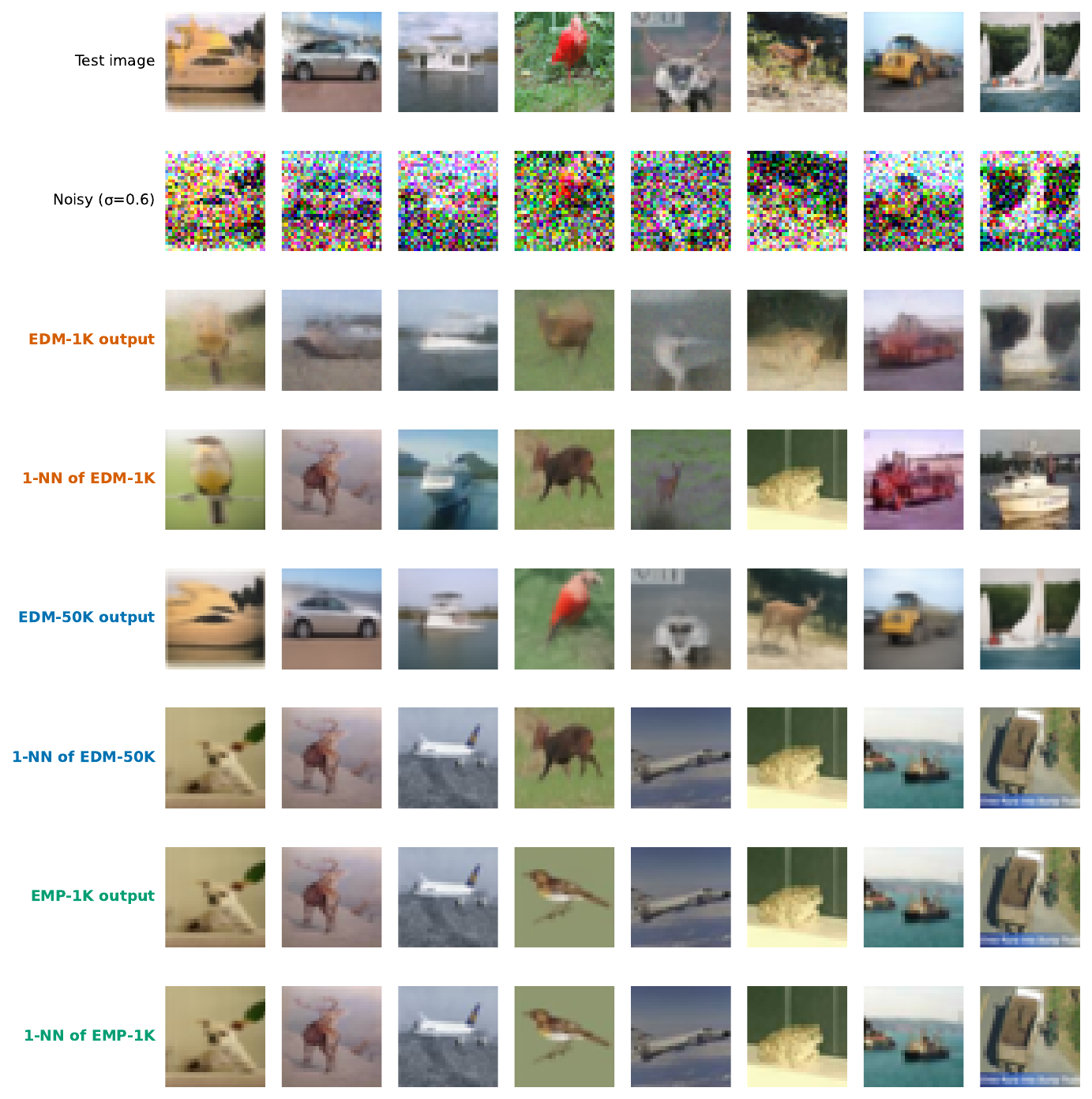}
    \caption{\textbf{One-Step Denoising at $\sigma = 0.6$.} From top: initial test image, noisy input, \textbf{EDM-1K} output, 1-NN of \textbf{EDM-1K}, \textbf{EDM-50K} output, 1-NN of \textbf{EDM-50K}, \textbf{EMP-1K} output, 1-NN of \textbf{EMP-1K}. At this lower noise level, \textbf{EDM-1K} memorization weakens while \textbf{EMP-1K} shows stronger nearest-neighbor retrieval.}
    \label{fig:mem_onestep_showcase_0.6}
\end{figure}

\begin{figure}[htbp]
    \centering
    \includegraphics[width=\linewidth]{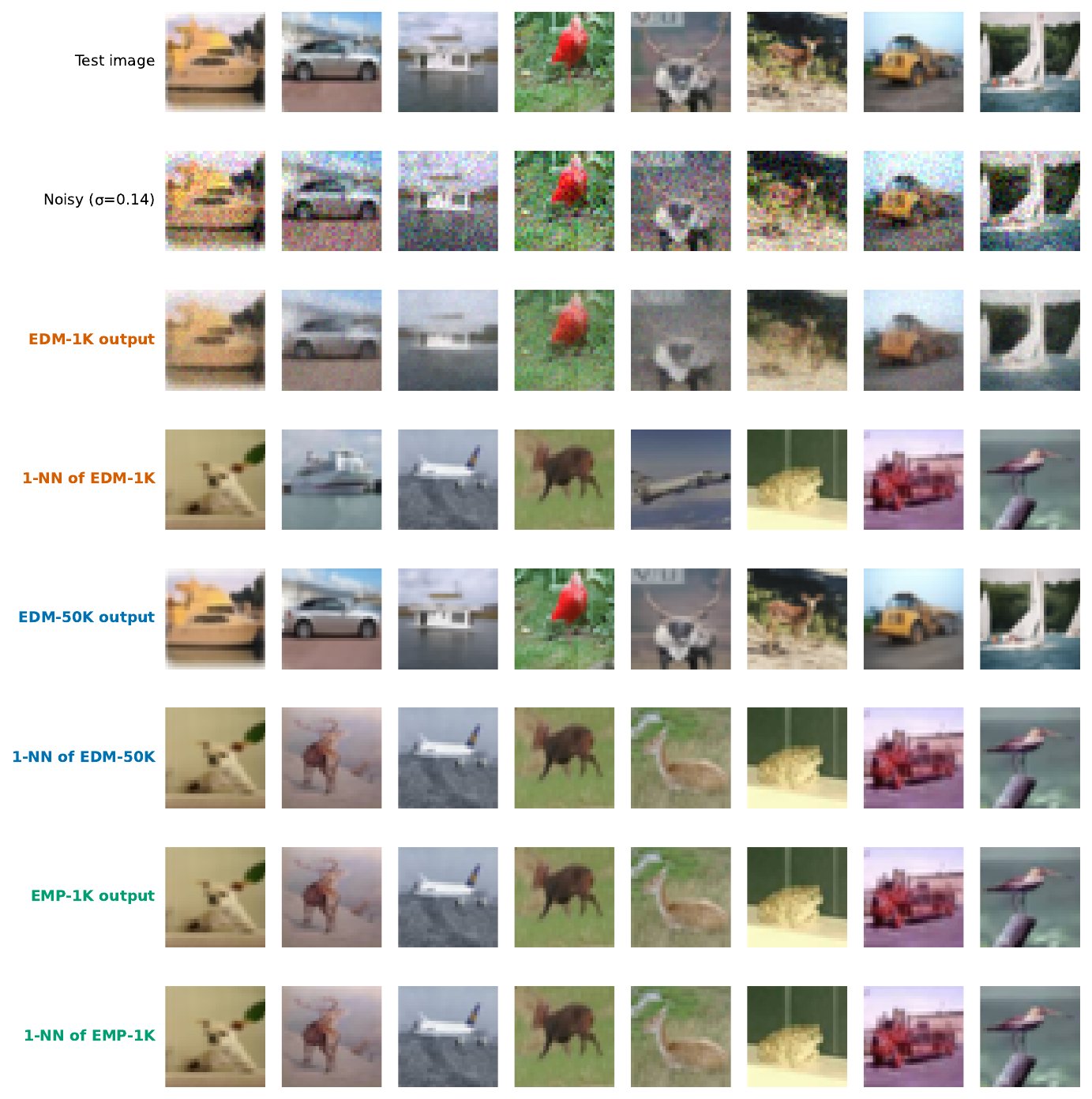}
    \caption{\textbf{One-Step Denoising at $\sigma = 0.14$.} From top: initial test image, noisy input, \textbf{EDM-1K} output, 1-NN of \textbf{EDM-1K}, \textbf{EDM-50K} output, 1-NN of \textbf{EDM-50K}, \textbf{EMP-1K} output, 1-NN of \textbf{EMP-1K}. At this low noise level, \textbf{EDM-1K} shows minimal memorization while \textbf{EMP-1K} exhibits nearly perfect nearest-neighbor behavior.}
    \label{fig:mem_onestep_showcase_0.14}
\end{figure}

\subsection{Detailed Memorization Visualizations for Gap Training}
\label{app:gap-training-viz}

We provide detailed visualizations comparing the memorization behavior of the three anti-memorization methods evaluated in Section~\ref{sec:experiments}.

Each image pair is arranged as [Generated sample | 1-NN from training set], where the left image shows a generated sample and the right image shows its nearest neighbor from the 2,000-image training set. A yellow border around both images in a pair indicates that the sample is classified as memorized according to the criterion: $\|g - x_{\text{1-NN}}\|_2 < \|g - x_{\text{2-NN}}\|_2 / 3$.
Figures~\ref{fig:gap-training-baseline}, \ref{fig:gap-training-dummy}, and \ref{fig:gap-training-gap} show complete 8$\times$8 grids with 64 generated samples for each method, enabling detailed assessment of memorization patterns.

\begin{figure}[h]
    \centering
    \includegraphics[width=0.9\linewidth]{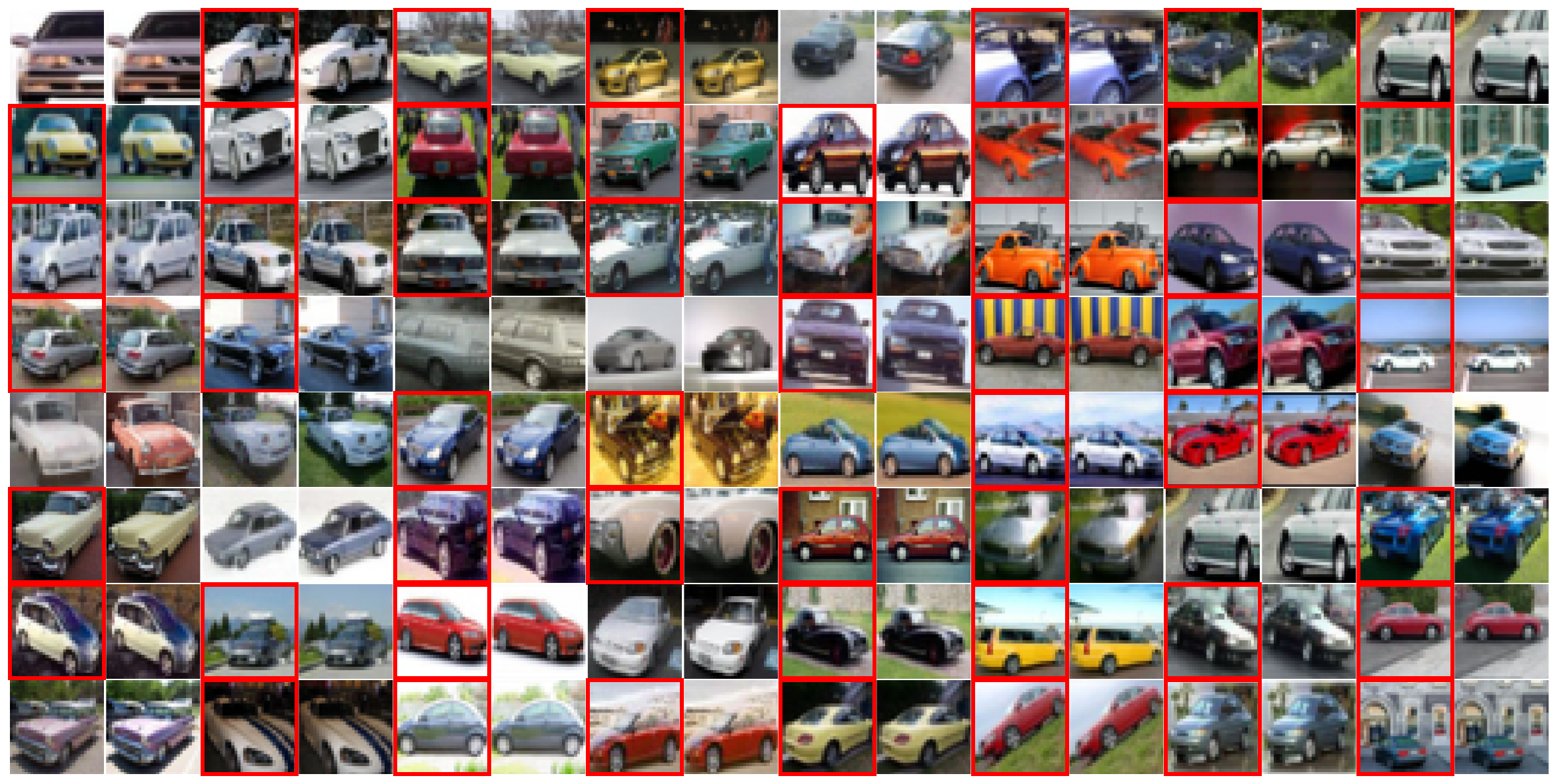}
    \caption{\textbf{2k-Baseline detailed visualization.} This 8$\times$8 grid shows 64 image pairs, each arranged as [Generated | 1-NN from training]. In each pair, the generated sample appears on the \emph{left} and its nearest neighbor from the 2,000-image training set appears on the \emph{right}.
        Red borders indicate memorized pairs. The high prevalence of red borders confirms severe memorization in the baseline case.}
    \label{fig:gap-training-baseline}
\end{figure}

\begin{figure}[h]
    \centering
    \includegraphics[width=0.9\linewidth]{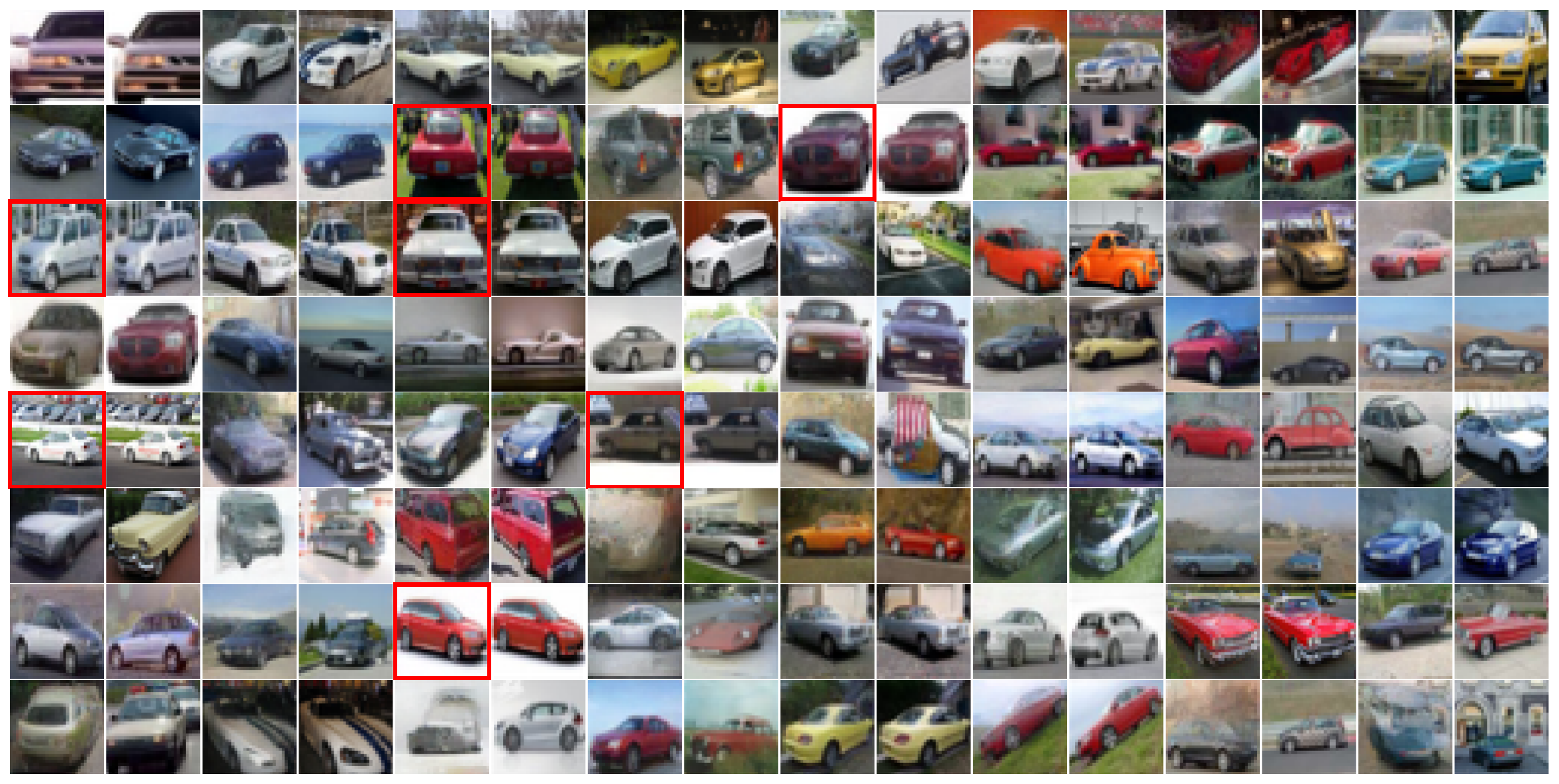}
    \caption{\textbf{2k-Dummy detailed visualization.} Same layout as Figure~\ref{fig:gap-training-baseline}: each of the 64 pairs shows [Generated | 1-NN]. Red borders mark memorized pairs. Conditional training with an auxiliary noise class substantially reduces memorization. Most generated samples now deviate meaningfully from their nearest neighbors, though several memorized samples remain (red borders).}
    \label{fig:gap-training-dummy}
\end{figure}

\begin{figure}[h]
    \centering
    \includegraphics[width=0.9\linewidth]{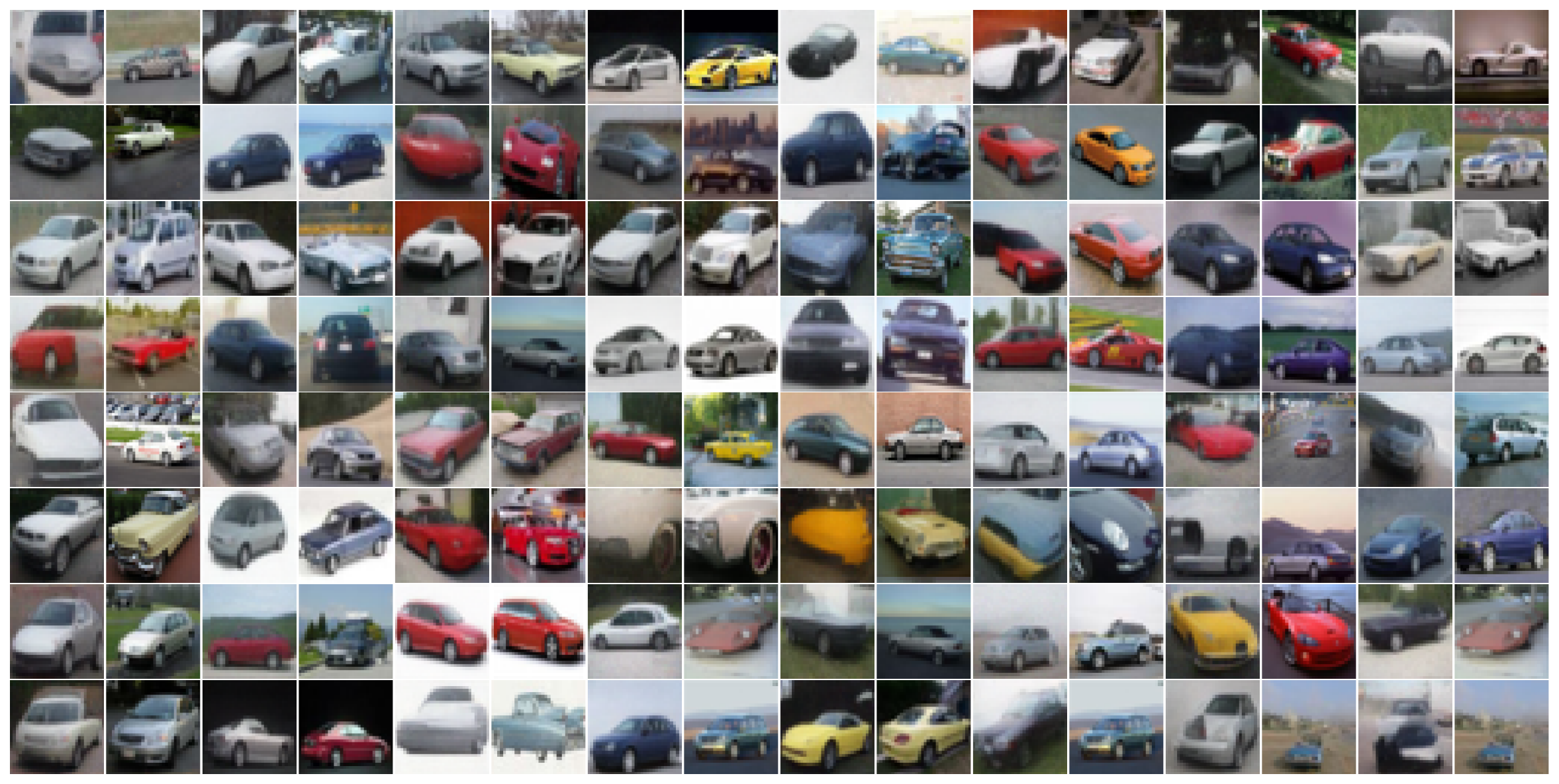}
    \caption{\textbf{2k-Gap detailed visualization.} Same layout: 64 pairs of [Generated | 1-NN]. Gap training (excluding $\sigma \in [1.0, 5.0]$ from training) achieves minimal memorization. Generated samples maintain semantic similarity to cars while introducing novel variations in color, viewpoint, and detail. This demonstrates that targeted undertraining in the medium-$\sigma$ danger zone enables robust generalization.}
    \label{fig:gap-training-gap}
\end{figure}

\subsection{Timestep Gap Configurations: Visual Comparisons}
\label{app:timestep-gap-viz}

Visual comparison of memorization behavior for different timestep gap configurations on 1024 grayscale CelebA samples. Each 4$\times$4 grid shows 16 image pairs arranged as [Generated | 1-NN from training]. Red borders indicate memorized samples according to the criterion: $\|g - x_{\text{1-NN}}\|_2 < \|g - x_{\text{2-NN}}\|_2 / 3$.

\begin{figure}[htbp]
    \centering
    \includegraphics[width=0.9\linewidth]{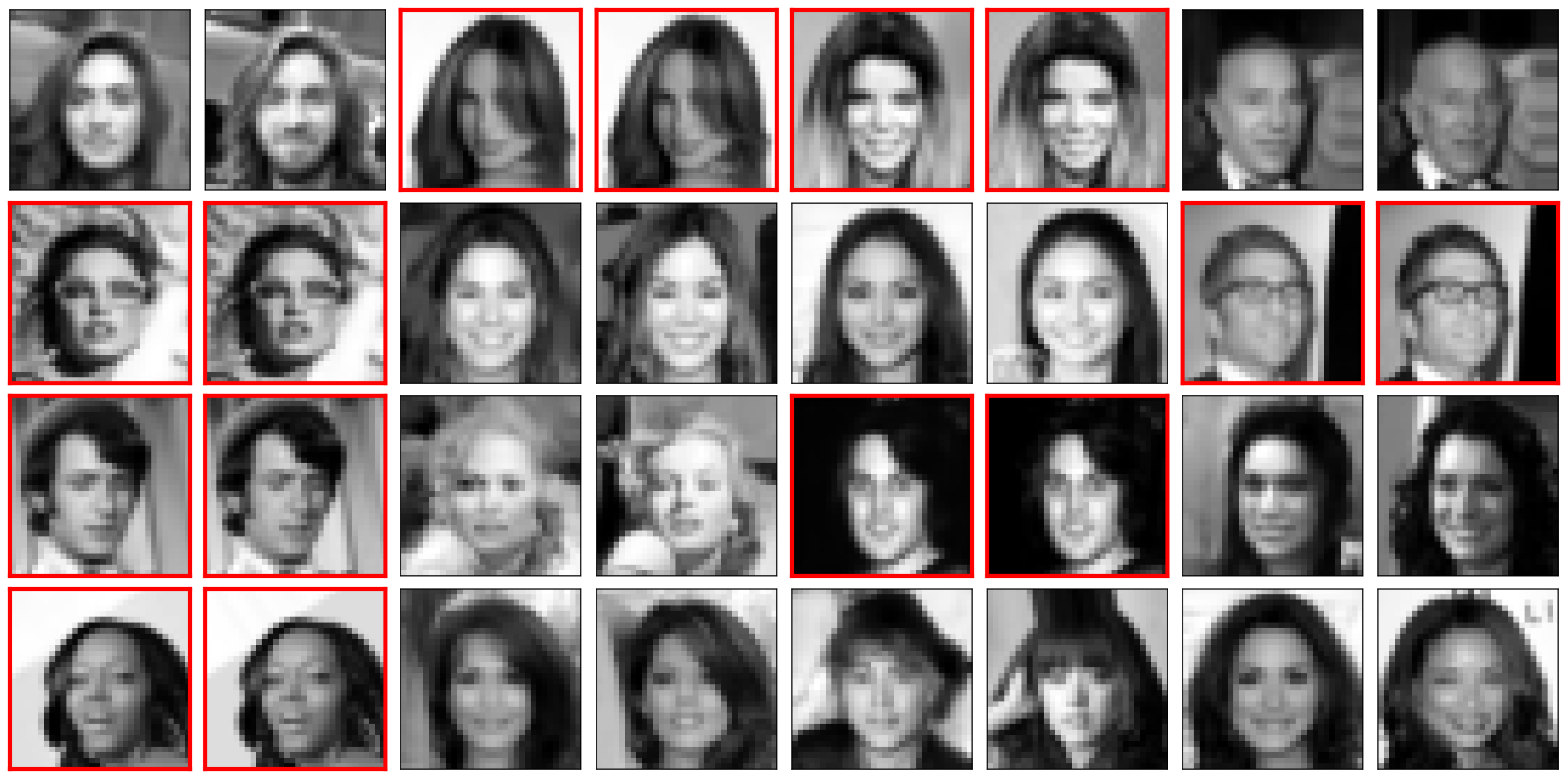}
    \caption{\textbf{Baseline (no gap).} 74.97\% memorization. The high prevalence of red borders confirms severe memorization in this low-data regime.}
    \label{fig:celeba-baseline}
\end{figure}

\begin{figure}[htbp]
    \centering
    \includegraphics[width=0.9\linewidth]{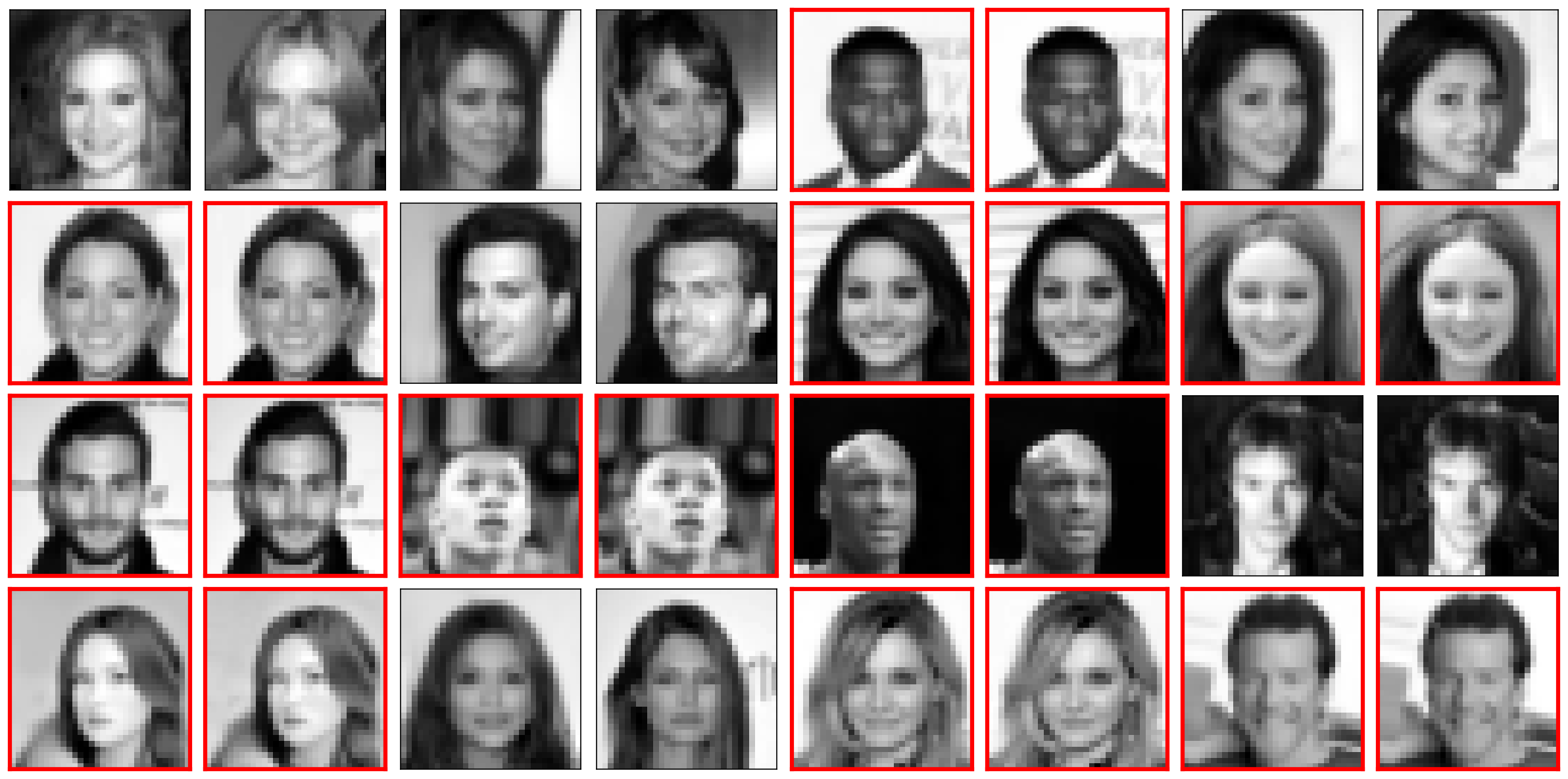}
    \caption{\textbf{Gap [100--200].} 42.93\% memorization. Skipping 100 timesteps in the early danger zone reduces memorization substantially.}
    \label{fig:celeba-gap-100-200}
\end{figure}

\begin{figure}[htbp]
    \centering
    \includegraphics[width=0.9\linewidth]{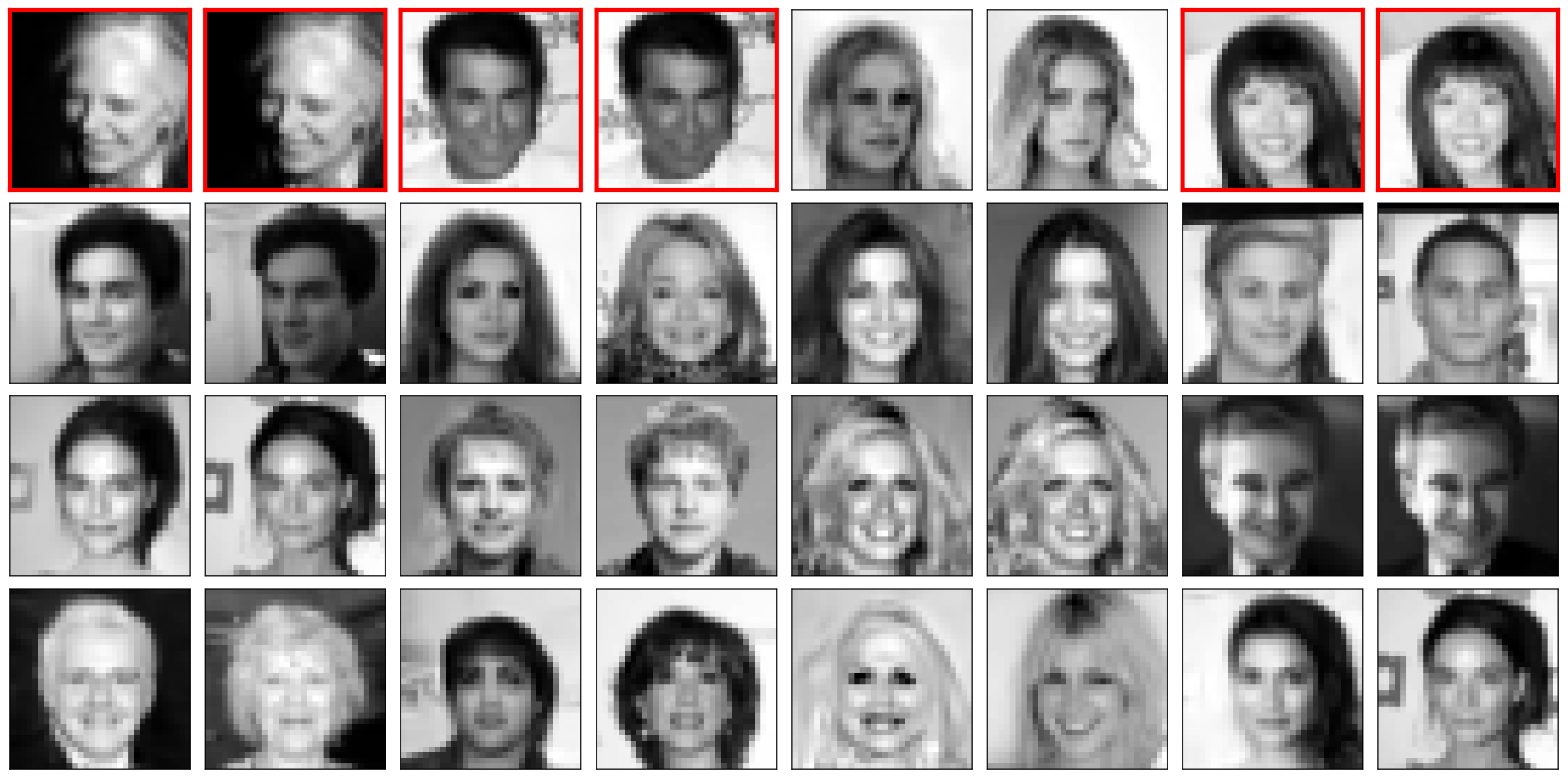}
    \caption{\textbf{Gap [100--275].} 13.45\% memorization. Extending the gap further into the danger zone dramatically reduces memorization.}
    \label{fig:celeba-gap-100-275}
\end{figure}

\begin{figure}[htbp]
    \centering
    \includegraphics[width=0.9\linewidth]{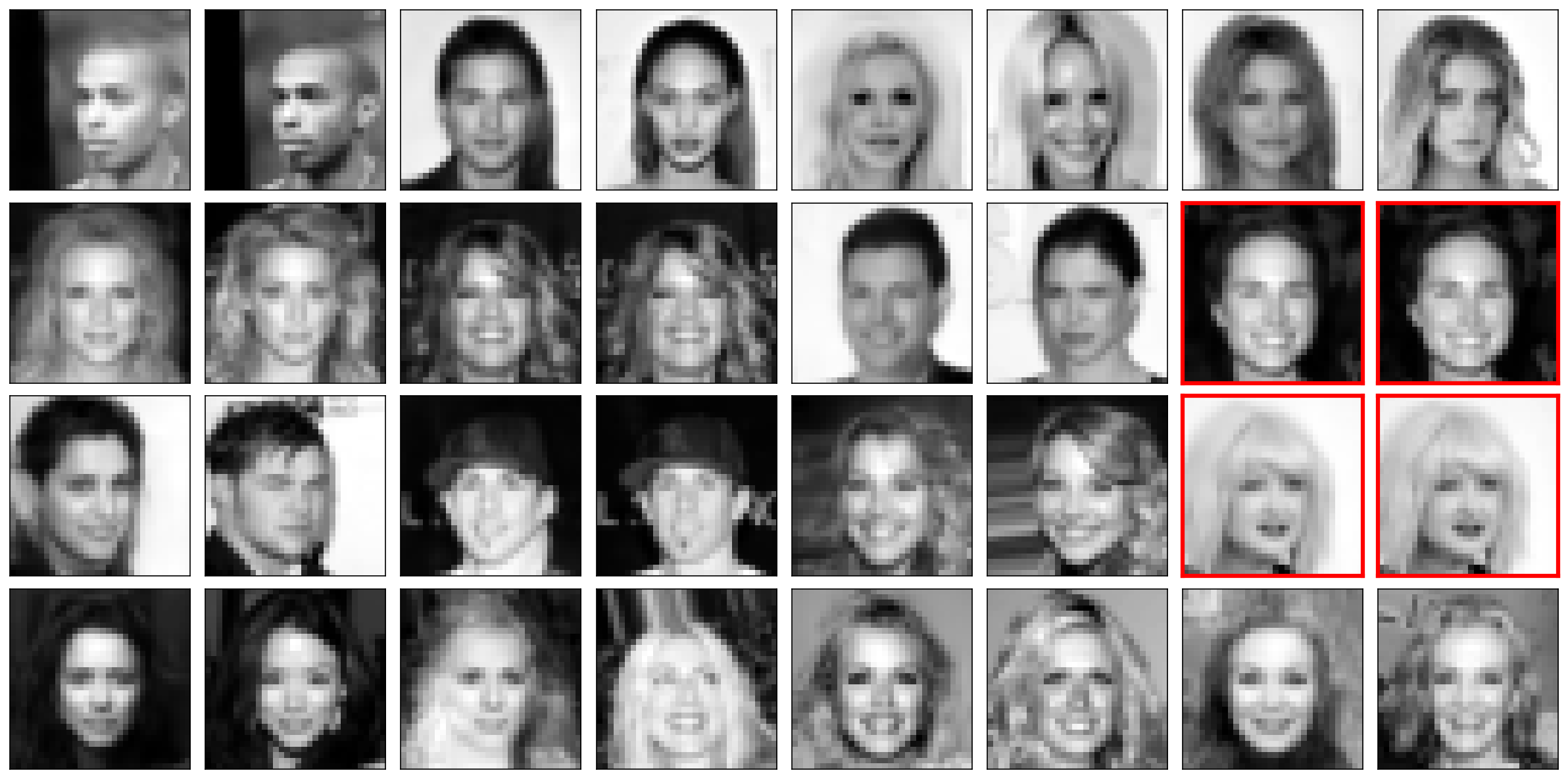}
    \caption{\textbf{Gap [100--350].} 9.24\% memorization. The widest gap achieves the lowest memorization rate, though at the cost of increased FID (26.74 vs 13.09 baseline).}
    \label{fig:celeba-gap-100-350}
\end{figure}

\end{document}